\documentclass[letterpaper]{article}

\usepackage{aaai25}  
\usepackage{times}  
\usepackage{helvet}  
\usepackage{courier}  
\usepackage[hyphens]{url}  
\usepackage{graphicx} 
\usepackage{natbib}  
\usepackage{caption} 
\usepackage{algorithm}
\usepackage{algorithmic}

\usepackage{newfloat}
\usepackage{listings}
\DeclareCaptionStyle{ruled}{labelfont=normalfont,labelsep=colon,strut=off} 
\floatstyle{ruled}
\newfloat{listing}{tb}{lst}{}
\floatname{listing}{Listing}
\usepackage{xcolor} 
\definecolor{titleblue}{RGB}{51,51,178}
\usepackage{hyperref}
\hypersetup{
    colorlinks=true,       
    linkcolor=titleblue,        
    citecolor=titleblue,       
    filecolor=titleblue,     
    urlcolor=titleblue         
}

\usepackage{amssymb}
\usepackage{natbib}
\usepackage{color}

\usepackage{amsfonts}       
\usepackage{url}
\usepackage{amsmath}
\usepackage{amssymb}
\usepackage{amsthm}
\usepackage{mathtools}
\usepackage{tikz}
\usepackage{subfig}
\usepackage{algorithmic}
\usepackage{subfig}
\usepackage{footmisc}
\usepackage{bibentry}
\nobibliography*
\usepackage{mathtools}

\usepackage{physics}
\usepackage{mathrsfs}
\mathtoolsset{showonlyrefs}
\usepackage{etoolbox}
\usepackage{makecell}
\usepackage{enumerate}
\usepackage{booktabs}
\usepackage{lastpage}
\usepackage{float}

\usepackage{thmtools, thm-restate}

\newcommand{\explain}[1]{\tag*{(#1)}}
\newcommand{\explaind}[2]{\makebox[0.85\textwidth]{$\displaystyle#1$\hfill(#2)}}

\newcommand{\fS}{\mathcal{S}}
\newcommand{\fA}{\mathcal{A}}

\newcommand{\R}{\mathbb{R}}

\newcommand{\E}{\mathbb{E}}
\newcommand{\V}{\mathbb{V}}

\newcommand{\na}{{|\fA|}}

\newcommand{\pdisgpik}{G^{\text{PDIS}}_k}
\newcommand{\rhok}{\rho^{\pik,\mu}}
\newcommand{\rhoks}{\rho^{\pik,\mus}}

\newcommand{\Asi}{A^{[\mus,i]}}

\newcommand{\Aki}{A^{[\pik,i]}}

\newcommand{\pik}{\pi^{(k)}}

\newcommand{\mus}{\mu^{*}}

\newcommand{\etak}{\eta^{\qty(k)}}

\newcommand{\wj}{w^{(j)}}
\newcommand{\wk}{w^{(k)}}
\newcommand{\barw}{\bar{w}}

\newcommand{\Eonk}{E^{\text{on}, \pik}}
\newcommand{\Eoffk}{E^{\text{off}, \pik}}

\newcommand{\tb}[1]{{\textbf{#1}}}

\newtheorem{theorem}{Theorem}
\newtheorem{lemma}{Lemma}

\newtheorem*{theorem*}{Theorem}
\newtheorem*{proposition*}{Proposition}
\newtheorem*{lemma*}{Lemma}
\newtheorem*{corollary*}{Corollary}
\newtheorem*{assumption*}{Assumption}
\newtheorem*{definition*}{Definition}
\newtheorem*{remark*}{Remark}

\allowdisplaybreaks

\title{Efficient Multi-Policy Evaluation for Reinforcement Learning}

\author {
    Shuze Daniel Liu\textsuperscript{\rm 1},
    Claire Chen\textsuperscript{\rm 2},
    Shangtong Zhang\textsuperscript{\rm 1}
}
\affiliations {
\textsuperscript{\rm 1}Department of Computer Science, University of Virginia
\\
\textsuperscript{\rm 2}School of Arts and Science, University of Virginia\\
shuzeliu@virginia.edu, clairechen@email.virginia.edu, shangtong@virginia.edu
}

\begin{document}

\maketitle

\begin{abstract}
To unbiasedly evaluate multiple target policies, the dominant approach among RL practitioners is to run and evaluate each target policy separately. However, this evaluation method is far from efficient because samples are not shared across policies, and running target policies to evaluate themselves is actually not optimal. In this paper, we address these two weaknesses by designing a tailored behavior policy to reduce the variance of estimators across all target policies. Theoretically, we prove that executing this behavior policy with manyfold fewer samples outperforms on-policy evaluation on every target policy under characterized conditions. Empirically, we show our estimator has a substantially lower variance compared with previous best methods and achieves state-of-the-art performance in a broad range of environments.
\end{abstract}

\section{Introduction} 
We explore the multi-policy evaluation problem, where we aim to estimate the performance of multiple target policies.
In reinforcement learning (RL, \citet{sutton2018reinforcement}), 
multi-policy evaluation is prevalent for model selections \citep{schulman2017proximal, prechelt2002early}. A simple method to evaluate multiple policies is to perform online policy evaluation for each target policy separately. However, the number of required online samples scales up quickly with the number of target policies.

In many scenarios, heavily relying on massive online data is not preferable. Firstly, collecting massive online data can be both expensive and slow when interacting with the real world \citep{li2019perspective, Zhang_2023, chen2024efficient, liu2024doubly}. Secondly, even if there is a simulator, for complex problems such as data center cooling, each step may still cost 10 seconds \citep{chervonyi2022semi}. Thus, building an RL system demanding millions of steps remains expensive.

To address the expensive nature of online data, 
offline RL is proposed to mitigate the dependency on online data. However, RL systems built only on offline datasets have uncontrolled bias.  A policy showing high performance on offline data may actually perform very poorly in real deployment \citep{levine2018reinforcement}. Therefore, both online and offline RL practitioners heavily use online methods to evaluate the performance of policies. 

To efficiently evaluate multiple policies,
previous works try to reuse online samples generated by other target policies. \citet{agarwal2017effective} show that naively combining data generated by other policies may actually worsen the estimation. Data from other policies must be carefully reweighed before consideration. However, in multi-step reinforcement learning, those weights require knowing complicated covariance terms between every pair of target policies \citep{lai2020optimal}. Such strong prior knowledge is rarely available and makes these methods impractical. To avoid complex weights, other literature \citep{dann2023multiple} tried to reuse online samples by assuming deterministic policies and a flexible environment that can start from any desired state. These assumptions rarely hold. 

In our work, we design a tailored behavior policy to \textbf{efficiently} and \textbf{unbiasedly} evaluate all target policies. Our method does not require knowing any complex covariance and applies to general RL settings without any restrictive assumptions. 
Our contribution is two-fold.

Theoretically, our method is always unbiased (Theorem \ref{lemma: rl pdis unbaised}) and is proven to achieve lower variance than the on-policy estimator for each target policy under characterized conditions (Theorem \ref{theorem: better than each rl}, Theorem \ref{theorem: better than each average rl}). Moreover, we introduce a similarity metric between policies and prove that the number of required samples for our method does not scale with the number of target policies under the similarity condition.

Empirically, compared with previous best methods,  we show our estimator has a substantially lower variance. Our method requires much fewer samples to reach the same level of accuracy and 
achieves state-of-the-art performance in a broad range of environments. 
 
\section{Related Work}

\label{sec related work}
\tb{Multiple target policies.} 
In multi-policy evaluation, traditional approaches often evaluate each policy separately using on-policy Monte Carlo methods. However, this ordinary method ignores the potential similarity between target policies and is crude for two reasons.
\textbf{First}, the method does not utilize data sampled by other policies,
causing the number of required online samples to scale quickly with the number of target policies.
\textbf{Second}, even for a single target policy, the on-policy evaluation method is still not the optimal choice.
Through a tailored behavior policy \citep{liu2024efficient}, the variance of the on-policy Monte Carlo evaluation can be reduced while achieving an unbiased estimation. 

To address the inefficiency in multi-policy evaluation problem,
\citet{dann2023multiple} present an algorithm to reuse online samples from target policies. However, their algorithm works only when all target policies are deterministic, which is also highly restricted. \textit{By contrast, our method copes with stochastic policies.} The key difference is that
they consider the plain approach by reusing samples from target policies, while we propose a tailored behavior for multiple target policies, which is designed to generate samples that all similar policies can efficiently share. \citet{liu2024efficient} also design a behavior policy for off-policy evaluation. However, they only consider a single target policy, which is narrower than our settings. In the empirical results section, we also show that our method outperforms theirs \citep{liu2024efficient}, in multi-policy evaluation problems under the multi-step RL setting. Using a shared behavior policy tailored for all similar target policies, our method achieves state-of-the-art performance and does better than all existing methods.

\tb{Multiple logging policies.}
Other approaches consider using data from multiple logging policies to perform off-policy evaluation, although only aiming at a single target policy. We call them logging policies because in their works \citep{agarwal2017effective, lai2020optimal,kallus2021optimal}, data are previously logged from certain behavior policies and are fixed. This is different from our setting, in which we design an active data-collecting policy for multiple target policies.

\citet{agarwal2017effective} point out that directly combining data from different policies may increases the estimation variance. They then propose two new estimators by reweighting data from different policies. However, their method is restricted to the contextual bandit setting. \textit{By contrast, we work on the multi-step reinforcement learning setting, which is much broader.} \citet{lai2020optimal} extend the method from \citet{agarwal2017effective} into multi-step RL. Nevertheless, getting the desired weights for different logging policies requires knowing complicated covariance terms between every pair of logging policies. That is, given $K$ logging policies, their method needs to compute $K^2$ covariances. Such strong prior knowledge is rarely available and is computationally expensive, making the method impractical. Furthermore, they ignore bias from any off-policy estimator.\textit{ By contrast, with the tailored behavior policy, our estimator is inherently and provably unbiased.} \citet{kallus2021optimal} also explore off-policy evaluation with multiple target policies in RL setting.
They combine the reweighting strategy with the control variate method, leading to a reduced variance estimation. However, getting the weights proposed by their method requires knowledge of state visitation densities, whose approximation is very challenging in MDPs with large stochasticity and function approximation (cf. model-based RL \citep{sutton1990integrated,sutton2012dyna,deisenroth2011pilco,chua2018deep}). Due to this impracticability, \citet{kallus2021optimal} only conduct experiment of their method in the contextual bandit setting, remaining the experiment on the multi-step RL setting untouched. \textit{By contrast, our method avoids reliance on any terms that are impractical to estimate.}

\section{Background}
In this work, we focus on a finite horizon Markov Decision Process (MDP), as defined by \citet{puterman2014markov}, with a finite state space $\fS$, a finite action space $\fA$, 
a reward function $r: \fS \times \fA \to \R$,
a transition probability function $p: \fS \times \fS \times \fA \to [0, 1]$,
an initial distribution $p_0: \fS \to [0, 1]$,
and a constant horizon length $T$.   
To simplify notations, we consider the undiscounted setting. But our results can be naturally applied to the discounted setting \citep{puterman2014markov} as long as the horizon is fixed and finite.
We define $[n] \doteq \qty{0, 1, \dots, n}$ for any integer $n$.

There are $K$ policies to be evaluated. In this paper, any index with parenthesis around it (e.g. $\pik$) is related to the \emph{policy index}. We define abbreviations $\pik_{i:j} \doteq \qty{\pik_i, \pik_{i+1}, \dots, \pik_j}$ and $\pik \doteq \pik_{0:T-1}$, where $\pi_t^{(k)}: \fA \times \fS \to [0, 1]$ defines the probability of selecting action $A_t$ given the state $S_t$ at time $t \in [T-1]$. An initial state $S_0$ is sampled from $p_0$ at time step $0$. At each time step $t$, after the execution of an action, a finite reward $R_{t+1} = r(S_t, A_t)$ is obtained and a successor state $S_{t+1}$ is sampled from $p(\cdot \mid S_t, A_t)$.

We define the return at time step $t$ as 
$
  G_t \doteq \sum_{i={t+1}}^T R_i.
$
The state- and action-value function is defined as
$
v_{\pik, t}(s) \doteq \E_{\pik}\left[G_t \mid S_t = s\right]$ and $
q_{\pik, t}(s, a) \doteq \E_{\pik}\left[G_t \mid S_t = s, A_t = a\right].
$
The performance of the policy $\pi$ is defined as $\textstyle J(\pik) \doteq \textstyle\sum_s p_0(s) v_{\pik, 0}(s).$
We adopt the total rewards performance metric, introduced by \citet{puterman2014markov}, as a measurement
of the performance. In this work, we focus on the Monte Carlo methods, which have been widely adopted since their introduction by \citet{kakutani1945markoff}. 
We draw samples of $J(\pik)$ by executing the policy $\pik$ online. The empirical average of the sampled returns converges to $J(\pik)$ as the number of samples increases.
Since this method estimates the performance of a policy $\pik$ by running itself, it is called on-policy learning (\citealt{sutton1988learning}).

Henceforth, we study off-policy learning,
in which we need to estimate the total rewards $J(\pik)$ of a policy $\pik$,
called the target policy,
by running a different policy $\mu$,
known as the behavior policy. Each trajectory
$
\textstyle \qty{S_0, A_0, R_1, S_1, A_1, R_2, \dots, S_{T-1}, A_{T-1}, R_T}
$
is generated by a behavior policy $\mu$ with
$
  S_0 \sim p_0, A_{t} \sim \mu_{t}(\cdot | S_{t}), \, t  \in [T-1].
$
We use
$$
  \tau^{\mu_{t:T-1}}_{t:T-1} \doteq \qty{S_t, A_t, R_{t+1}, \dots, S_{T-1}, A_{T-1}, R_{T}}
$$
to denote a segment of a random trajectory generated by the behavior policy $\mu$ from the time step $t$ to the time step $T-1$ inclusively.  
The key tool for off-policy learning is importance sampling (IS) \citep{Rubinstein1981Simulation}, which is used to reweight rewards collected by $\mu$ to give an unbiased estimate of $J(\pik)$. For each policy $\pik$, 
the importance sampling ratio at time step $t$ is defined as
$
\textstyle \rhok_t \doteq \frac{\pik_t(A_t \mid S_t)}{\mu_t(A_t \mid S_t)}.
$
Then, the product of importance sampling ratios from time $t$ to $t' \geq t$ is defined as 
$
\textstyle \rhok_{t:t'} \doteq \prod_{i=t}^{t'} \frac{\pik_i(A_i | S_i)}{\mu_i(A_i | S_i)}.
$
Among the various ways to use the importance sampling ratios in off-policy learning \citep{geweke1988antithetic, hesterberg1995weighted, koller2009probabilistic,thomas2015safe},
we use the per-decision importance sampling estimator (PDIS, \citet{precup:2000:eto:645529.658134}) in this paper and leave the study of others for future work.
For $\pik$, the PDIS Monte Carlo estimator is defined as  $\textstyle \pdisgpik\qty(\tau^{\mu_{t:T-1}}_{t:T-1}) \doteq \sum_{i=t}^{T-1} \rhok_{t:i} R_{i+1}$, which is unbiased
for any behavior policy $\mu$  that covers target policy $\pik$ \citep{precup:2000:eto:645529.658134}. 
That is,
when $\forall s$, $\forall a$,
$\mu_t(a|s) = 0 \implies \pik_t(a|s)=0$, we have $\forall t$, $\forall s$,
$\E[ \pdisgpik\qty(\tau^{\mu_{t:T-1}}_{t:T-1}) \mid S_t = s ]  = v_{\pik,t}(s)$.
We also leverage the recursive form of the PDIS estimator:
\begin{align}\label{eq: PDIS-recursive}
&\pdisgpik\qty(\tau^{\mu_{t:T-1}}_{t:T-1}) \\
=&\begin{cases}
\rhok_t \left(R_{t+1} + \pdisgpik\qty(\tau^{\mu_{t+1:T-1}}_{t+1:T-1})\right) & t \in [T-2], \\
\rhok_tR_{t+1} & t = T-1.
\end{cases}
\end{align}
Because the PDIS estimator is unbiased,
reducing its variance is sufficient for the improvement of its sample efficiency. 
We achieve this variance reduction goal for multiple policies by designing a tailored behavior policy.

\section{Variance Reduction in Statistics} \label{sec:var-stats}
In this section, we propose the mathematical framework for variance reduction using importance sampling ratios.
Let $A$ be a discrete random variable with a finite set of possible values $\fA$, and assume it follows a probability mass function $\pik:\fA \to [0,1]$, called target policy. Additionally, let $q:\fA \to \R$ be a function that maps elements of $\fA$ to real numbers. Our objective is to estimate $\E_{A\sim \pik}[q(A)]$ for each $\pik$, where $k$ is an index within a finite set $[K]$. Since in this paper, data can be generated from multiple distributions, we specify their source clearly. We reserve the superscript with brackets $[\cdot,\cdot]$ to denote the source and the index of samples. For example, $A^{[\pik, i]}$ is the $i$th sample generated by running $\pik$. We use $n_k$ to denote the total number of samples sampled by policy $\pik$.
The plain Monte Carlo methods then samples $\qty{A^{[\pik, 1]}, \dots, A^{[\pik, n_k]}}$ from each
$\pik$ 
  and use the empirical average
$\frac{1}{n_k}\sum_{i=1}^{n_k} q(A^{[\pik, i]})$
as the estimate for each $\E_{A\sim \pik}[q(A)]$. 

The importance sampling is introduced as a variance reduction technique in statistics, where the main idea is 
to sample $\qty{q(A^{[\mu,i]})}_{i=1}^N$ following a distribution $\mu$ and use 
 $\frac{1}{N}\sum_{i=1}^N \rhok(A^{[\mu,i]}) q(A^{[\mu,i]})$
as the estimate,
where 
$
\textstyle  \rhok(A) \doteq \frac{\pik(A)}{\mu(A)}
$
is the importance sampling ratio. In this statistics section, we propose the optimal behavior policy $\mu$ that evaluates all target policies $\pik$ simultaneously by sharing samples. We also define the similarity of policies and prove when target policies satisfy the similarity condition, samples needed to estimate all of them do not scale with the number of policies $K$. These ideas are later extended into the multi-step reinforcement learning (RL) setting in the following section.

Assuming that $\forall i, \mu$ covers $\pik$, i.e.,
\begin{align}
\label{eq: stats converage}
\forall a, \mu(a) = 0 \implies \pik(a) = 0.   
\end{align}
Then, the importance sampling ratio weighted empirical average is unbiased, i.e., $\forall k$,
$\E_{A\sim \pik}[q(A)] = \E_{A\sim \mu}[\rhok(A)q(A)].$
If we carefully design the sampling distribution $\mu$, 
the variance can be reduced. 
We formulate 
this problem of searching a variance-reducing sampling distribution for $K$ policies as an optimization problem
\begin{align} 
\text{min}_{\mu \in \Lambda_-}  \quad & \textstyle\sum_{k \in [K]}
\V_{A\sim \mu}(\rhok(A)q(A)) \label{eq: math-optimization1},
\end{align}
where $\Lambda_-$ is the classical search space \citep{Rubinstein1981Simulation, zhang2022thesis, liu2024ode, qian2024almost} defined as
\begin{align}
\Lambda_- \doteq \qty{\mu \in \Delta(\fA) \mid \forall a, \forall k, \mu(a) = 0 \Rightarrow \pik(a) = 0}.
\end{align}
Here, $\Delta(\fA)$ denotes the set of all probability distributions on the set $\fA$.
In other words,
$\Lambda_- $ includes all distributions that cover $\qty{\pik}_{k=1}^K$.
In this work,
we  enlarge $\Lambda_-$ to $\Lambda$, which is defined as
\begin{align}
\Lambda \doteq\qty{\mu \in \Delta(\fA) \mid \forall a, \forall k, \mu(a) = 0 \Rightarrow \pik(a)q(a) = 0}.
  \label{eq: stats search space}
\end{align}
The space $\Lambda$ weakens the assumption in \eqref{eq: stats converage}.
We prove that any distribution $\mu$ in $\Lambda$ still gives unbiased estimation,
though $\Lambda_-  \subseteq \Lambda$.

\begin{restatable}[]{lemma}{reOOlemmaOOstatsOOunbiasedness}
\label{lemma: stats unbiasedness}
$\forall \mu \in \Lambda, \forall k$,
\begin{align}
 \E_{A\sim\mu}\left[\rhok(A)q(A)\right] = \E_{A\sim\pik}\left[q(A)\right].     
\end{align}
\end{restatable}
Its proof is in the appendix. 
We now consider the variance minimization problem on $\Lambda$, i.e.,
\begin{align} 
\text{min}_{\mu \in \Lambda}  \quad &  \textstyle \sum_{k \in [K]}
\V_{A\sim \mu}(\rhok(A)q(A)) \label{eq: math-optimization}.
\end{align}
The following lemma gives an optimal solution $\mu^*$ to the optimization problem~\eqref{eq: math-optimization}.

\begin{restatable}[]{lemma}{reOOlemmaOOmathOOoptimal}
\label{lemma: math optimal}
Define $\mu^*(a) \propto \sqrt{\sum_{k \in [K]}{\pik}(a)^2q(a)^2}$.
Then $\mu^*$ is an optimal solution to \eqref{eq: math-optimization}.
\end{restatable}

Its proof is in the appendix. Here,  $\mu(a) \propto  f(a)$ with some non-negative $f(a)$ means
\begin{align}
\textstyle \mu(a) \doteq  f(a) / \sum_b  f(b).
\end{align}
If $f(a) = 0$ for all $a$,
the above ``reweighted'' distribution is not well defined.
We then use the convention to interpret $\mu(a)$ as a uniform distribution, i.e., $\mu(a) = 1/\na$.
\emph{This convention in using $\propto$ is adopted in the rest of the paper for simplicity.}

When estimating $\E_{A \sim \pik}\qty[q(A)]$, $\pik(a)q(a)$ shows how much an action contributes to the expectation and is heavily used \citep{owen2013Monte, liu2024efficient}. 
Denote 
\begin{align}
\wk(a) &\doteq \qty(\pik(a)q(a))^2,  \label{def: w} \\
\barw(a) & \doteq \textstyle\sum_{j \in [K]}{\wj}(a)/K. \label{def: barw}
\end{align}
We use $\etak(a)$ to denote the similarity between $\pik$ and the average $\barw(a)$,
\begin{align}\label{def: eta}
\etak(a)  \doteq  \wk(a)/\bar{w}(a).
\end{align}
Naturally, $\etak(a) = 1$  when all policies are the same on $a$.
Define $\underline{\eta} \doteq \min_{k,a} \etak(a)$ and  $\overline{\eta} \doteq \max_{k,a} \etak(a)$, we have $\forall k, a$, 
\begin{align}\label{eq: eta inequality}
\underline{\eta} \leq \etak(a) \leq \overline{\eta}.
\end{align}
In the following theorem, we compare the variance of estimation methods.
For off-policy evaluation, our designed $\mus$ generates $n$ samples. For on-policy evaluation, when evaluating multiple policies, it is common for different policies to generate different numbers of samples. Thus, to achieve a fair and general enough comparison, each target policy $\pik$ generates $n_k$ samples. There is no constraint on $n_k$, as long as $ \sum_{k=1}^K n_k=n$. Using $A^{[\pik, i]}$ to denote the $i$th sample generated following $\qty{\pik}$, we define the empirical average for all $\pik$ as
\begin{align}
\Eonk \doteq & \textstyle\frac{\sum_{i=1}^{n_k}q(\Aki)}{n_k}. \label{def: Eonk}
\end{align}
Similarly, using $A^{[\mus, i]}$ to denote the $i$th sample generated by $\mus$, We define the 
empirical average for all $\pik$ as
\begin{align}
\Eoffk  \doteq \textstyle\frac{\sum_{i=1}^n  \rhoks(\Asi)q(\Asi)}{n}. \label{def: Eoffk}
\end{align}
Then, we characterize sufficient conditions on policy similarity such that with the same total samples, off-policy evaluation with our tailored behavior policy $\mu^*$ achieves a lower variance than on-policy Monte Carlo on each $\pik$.
\begin{restatable}[]{lemma}{reOOlemmaOObetterOOthanOOeachOOaverage}
\label{lemma: better than each average}
$\forall k \in [K]$,
\begin{align}
\V_{A \sim \mus}\qty(\Eoffk) \leq \V_{A \sim \pik}\qty(\Eonk),
\end{align}
if the similarity $\eta(\cdot)$ has $\forall k$,
\begin{align}
&\textstyle\sqrt{\frac{\overline{\eta}}{\underline{\eta}}} \qty(\sum_{a} \pik(a) q(a))^2 - \qty(\frac{n}{n_k} -1) \Delta^{(k)} \\ 
\leq &\textstyle\sum_{a} \pik(a) q(a)^2, \label{eq: eta sufficient average}
\end{align}
where 
\begin{align}\label{eq: def Delta k}
\textstyle\Delta^{(k)} \doteq \qty[ \sum_{a} \pik(a) q(a)^2 - \qty(\sum_{a} \pik(a) q(a))^2].    
\end{align}
\end{restatable}
Its proof is in the appendix. 
In Lemma~\ref{lemma: better than each average}, we show under characterized conditions, using only the same total samples $n$ generated by $\mu^*$, the off-policy estimator already achieves a lower variance than on-policy estimator for each target policy $\pik$. Now, we present a stronger lemma by allowing each target policy to also generate $n$ samples, resulting in a total of $nK$ samples, which is $K$ times larger than $n$.
Using the empirical average for on-policy estimator as defined in \eqref{def: Eonk}, we now have, for all $\pik$,
\begin{align}
\Eonk = &\textstyle\sum_{i=1}^{n}q(\Aki)/{n}. \label{def: Eonk_n}
\end{align}
Then, we simplify the variance of the on-policy estimator for $\pik$ as
\begin{align}
&\V_{A\sim \pik}(\Eonk)\\
=&\V_{A\sim \pik}(\textstyle \frac{ \sum_{i=1}^{n}q(\Aki)}{n}) \explain{By \eqref{def: Eonk_n}}\\
=&\frac{1}{n}\textstyle\V_{A\sim \pik}(\sum_{i=1}^{n}q(\Aki))\\
=&\V_{A\sim \pik}(q(A)).
\end{align}
In the last step, we leverage the independence of samples.
Similarly, using the definition of empirical average for off-policy estimator as defined in \eqref{def: Eoffk},
we have
\begin{align}
\V_{A\sim \pik}(\Eoffk) =   \V_{A \sim \mus}\qty(\rhoks(A)q(A)) .
\end{align}
Then, we formalize the superiority for the “$n$-to-$Kn$" comparison in the following theorem.

\begin{restatable}[]{lemma}{reOOlemmaOObetterOOthanOOeach}
\label{lemma: better than each}
$\forall k \in [K]$,
\begin{align}
\V_{A \sim \mus}\qty(\rhoks(A)q(A)) \leq \V_{A \sim \pik}\qty(q(A)),
\end{align}
if the similarity $\eta(\cdot)$ has $\forall k,$ 
\begin{align}\label{eq: eta sufficient}
\textstyle\sqrt{\frac{\overline{\eta}}{\underline{\eta}}} \qty(\sum_{a \in \fA} \pik(a)q(a))^2  \leq \sum_{a \in \fA} \pik(a) q(a)^2.
\end{align}
\end{restatable}
Its proof is in the appendix. The superiority of using our designed behavior policy $\mus$ comes from two sources. First, $\mus$ generates samples that all similar policies can efficiently share. Second, it is designed to generate low-variance and unbiased samples compared with the on-policy evaluation.

\section{Variance Reduction in Reinforcement Learning} \label{sec:variance}
We extend the techniques discussed in the statistics section into multi-step reinforcement learning (RL). In this section, Theorem~\ref{lemma: rl pdis unbaised} is the RL version of Lemma~\ref{lemma: stats unbiasedness} for unbiasedness. Theorem~\ref{lemma: rl-optimal} is the RL version of Lemma~\ref{lemma: math optimal} for behavior policy design. Theorem~\ref{theorem: better than each average rl} and \ref{theorem: better than each rl} are the RL version of Lemma~\ref{lemma: better than each average} and \ref{lemma: better than each}, respectively, for variance reduction.

As discussed in the related work section, the major caveat in multi-policy evaluation problems is data sharing. Without efficient data sharing, the total number of samples required for evaluating all policies increases rapidly with the number of target policies. Previous works try to reuse collected data across multiple target policies. However, their method rely on either (1) \textbf{restrictive assumptions}, namely, deterministic policies and flexible environment starting at any desired state \citep{dann2023multiple}, or (2) \textbf{impractical knowledge}, namely, complicated covariances \citep{lai2020optimal} and state visitation densities at very step \citep{kallus2021optimal}. Thus, none of the existing methods \citep{dann2023multiple, lai2020optimal, kallus2021optimal,agarwal2017effective} is implementable in the multi-step RL setting.

In this work, we tackle this notorious problem of efficient multi-policy evaluation in RL without any impracticability.
We seek to reduce the variance $\sum_{k\in [K]}\V\left(\pdisgpik\qty(\tau^{\mu_{0:T-1}}_{0:T-1})\right)$ 
by designing a proper behavior policy $\mu$.
Certainly,
we need to ensure that the PDIS estimator with this behavior policy is unbiased.

In the off-policy evaluation problem, classic reinforcement learning (\citealt{sutton2018reinforcement}) requires coverage assumption to ensure unbiased estimation. 
This means they only consider a set of  policies that cover $\qty{\pik}_{k=1}^K$, i.e.,
\begin{align}
\textstyle\Lambda_- \doteq 
\{& \mu \mid
\forall k, t, s, a, \mu_t(a|s) = 0 \implies \pik_t(a|s) = 0\}.
\end{align}
Similar to~\eqref{eq: stats search space},
we enlarge $\Lambda_-$ to
\begin{align}
\!\!\!\Lambda \doteq& \{\mu \mid \forall k, t, s, a, \mu_t(a|s) = 0 \\
&\implies 
\pik_t(a|s)q_{\pik, t}(s, a) = 0  \}.
\end{align}
We prove every policy $\mu \in \Lambda$ still achieves unbiased estimation in the following theorem. 
\begin{restatable}[Unbiasedness]{theorem}{reOOrlOOpdisOOunbaised}
\label{lemma: rl pdis unbaised}
$\forall \mu \in \Lambda$, $\forall k$, $\forall  t$, $\forall  s$, 
\begin{align}
    \E\left[\pdisgpik\qty(\tau^{\mu_{t:T-1}}_{t:T-1}) \mid S_t = s\right] = v_{\pik, t}(s).
\end{align}
\end{restatable}
Its proof is in the appendix.
One immediate consequence of Theorem~\ref{lemma: rl pdis unbaised} is that 
$
\forall \mu \in \Lambda, \forall k,
\E\left[\pdisgpik\qty(\tau^{\mu_{0:T-1}}_{0:T-1})\right] = J(\pik).
$
In this paper, 
we consider a set $\hat \Lambda$ such that $\Lambda_- \subseteq \hat \Lambda \subseteq \Lambda$.  $\hat \Lambda$ inherits the unbiasedness property of $\Lambda$ and is less restrictive than $\Lambda_-$,
the classical search space of behavior policies.
This $\hat \Lambda$ will be defined shortly. We now formulate our problem as 
\begin{align}
\label{eq: rl opt problem}
\textstyle \min_{\mu \in \hat \Lambda} \quad \sum_{k\in [K]}\V\left(\pdisgpik\qty(\tau^{\mu_{0:T-1}}_{0:T-1})\right).
\end{align}
By the law of total variance,
for any $\mu \in \hat \Lambda$,
we decompose the variance of the PDIS estimator as 
\begin{align}
\label{eq: varaince-1}
&\textstyle \sum_{k\in [K]}\V\left(\pdisgpik\qty(\tau^{\mu_{0:T-1}}_{0:T-1})\right) \\
=& \textstyle \sum_{k\in [K]}\E_{S_0}\left[\V\left(\pdisgpik\qty(\tau^{\mu_{0:T-1}}_{0:T-1}) \mid S_0\right)\right] \\
&+ \V_{S_0}\left(\E\left[\pdisgpik\qty(\tau^{\mu_{0:T-1}}_{0:T-1}) \mid S_0\right]\right) \\
=& \textstyle \sum_{k\in [K]}\E_{S_0}\left[\V\left(\pdisgpik\qty(\tau^{\mu_{0:T-1}}_{0:T-1}) \mid S_0\right)\right]\\
&+ \V_{S_0}\left(v_{\pik, 0}(S_0)\right)  \explain{by Theorem~\ref{lemma: rl pdis unbaised}}. 
\end{align}
The second term in~\eqref{eq: varaince-1} is a constant given a target policy $\pik$ and is unrelated to the choice of $\mu$. 
In the first term, 
the expectation is taken over $S_0$ that is determined by the initial probability distribution $p_0$. 
Consequently,
to solve the problem~\eqref{eq: rl opt problem},
it is sufficient to solve for each $s$,
\begin{align}
  \label{eq: rl opt problem2}
\textstyle \min_{\mu \in \hat \Lambda} \quad\sum_{k\in[K]} \V\left( \pdisgpik \qty( \tau^{\mu_{0:T-1}}_{0:T-1} ) \mid S_0 = s\right).
\end{align}

\color{black}
Denote the variance of the state value for the next state
given the current state-action pair $(s,a)$ as $\nu_{\pik, t}(s, a)$.
We have $\nu_{\pik, t}(s, a) = 0$ for $t = T-1$ and otherwise
\begin{align}\label{def:nu}
 \nu_{\pik,t}(s, a) \doteq\V_{S_{t+1}}\left(v_{\pik, t+1}(S_{t+1})\mid S_t=s, A_t=a\right).
\end{align}

To achieve variance reduction compared with on-policy evaluation, 
we aim to design $\hat \mu_t$ as an optimal solution to the following problem
\begin{align}\label{eq: one step optimization}
\textstyle\min_{\mu_t \in \hat{\Lambda}} \quad \sum_k \V\left(\pdisgpik\qty( \tau^\qty{\mu_t,\pik_{t+1}:\pik_{T-1}}_{t:T-1})\mid S_t = s\right),
\end{align}
The high-level intuition is that we aim to find the optimal behavior policy $\mu_t$ for the current step, assuming that in the future we perform the on-policy evaluation.
To define optimality, we first specify the set of policies we are concerned about.
To this end, we define that $\forall k$, $\hat q_{\pik, t}(s, a) \doteq q_{\pik, t}(s, a)^2$ for $t = T-1$ and otherwise
\begin{align}\label{eq: def q hat}
&\hat q_{\pik, t}(s, a) \doteq q_{\pik, t}(s, a)^2 + \nu_{\pik,t}(s, a) \\
 &+\textstyle\sum_{s'} p(s'|s, a)\V\left(\pdisgpik\qty(\tau^{\pik_{t+1:T-1}}_{t+1:T-1}) \mid S_{t+1} = s'\right).
\end{align}
Notably, 
$\hat q_{\pik, t}(s,a)$ is always \emph{non-negative} since all the summands are non-negative.
Accordingly,
we define $\hat \Lambda \doteq \{\mu \mid \forall k, t, s, a, \mu_t(a|s) = 0 \Rightarrow
\pik_t(a|s)\hat q_{\pik, t}(s, a) = 0\}$.
From \eqref{eq: def q hat},
we observe for any $k$, $t$, $s$, $a$,
$\hat q_{\pik, t}(s, a) \geq q_{\pik, t}(s, a) \geq 0$.
As a result, if $\mu_t \in \hat 
\Lambda$,
we have $\mu_t(a|s)=0\Rightarrow \pik_t(a|s)\hat q_{\pik, t}(s,a) = 0\Rightarrow \pik_t(a|s)q_{\pik, t}(s,a) = 0$.
Thus, $\hat \Lambda \subseteq \Lambda$.
To summarize,
we have $\Lambda_- \subseteq \hat \Lambda \subseteq \Lambda$.
$\hat \Lambda$ inherits the unbiased property of $\Lambda$ (Theorem~\ref{lemma: rl pdis unbaised}) and is larger than
the classic space $\Lambda_-$ considered in previous works \citep{precup:2000:eto:645529.658134,maei2011gradient,sutton2016emphatic,sutton2018reinforcement}.

Now, we define the optimal behavior policy as
\begin{align}
  \label{def hat mu}
\textstyle\hat \mu_t(a|s) \propto 
\sqrt{\sum_{k=1}^K {\pik_t}(a|s)^2\hat q_{\pik, t}(s, a)} .
\end{align}
$\hat q$ defined in \eqref{eq: def q hat} is different from $q$, and is always non-negative.
We confirm the optimality of $\hat\mu_t$ in the following theorem. 
\begin{restatable}[Behavior Policy Design]{theorem}{restaterloptimal}
\label{lemma: rl-optimal}
For any $k$, $t$ and $s$,
the behavior policy $\hat \mu_t(a|s)$ defined in \eqref{def hat mu} is an optimal solution to the following problem
\begin{align}
\!\!\!
\textstyle\min_{\mu_t \in \hat{\Lambda}} \quad \sum_k \V\left(\pdisgpik\qty( \tau^\qty{\mu_t,\pik_{t+1}:\pik_{T-1}}_{t:T-1})\mid S_t = s\right).
\end{align}
\end{restatable}

Its proof is in the appendix. Next, we formalize the similarity between target policies. Similar to \eqref{def: w}, \eqref{def: barw} in the statistics setting, 
$\forall k$,
$\forall t$, $\forall s$, we denote 
\begin{align}
\wk_t(s,a) &\doteq \pik_t(a|s)^2\hat q_{\pik,t}(s,a),  \label{def: w rl} \\
\textstyle\barw_t(s,a) & \textstyle\doteq \qty(\sum_{j \in [K]}\wj_t(s,a))/K. \label{def: barw rl}
\end{align}
Then, adopting the notation from \eqref{def: eta} and \eqref{eq: eta inequality}, we denote the similarity between $\pik_t$ and the average $\barw_t$ as
\begin{align}
\etak_t(s,a)  \doteq  \wk_t(s,a)/\bar{w}_t(s,a).\label{def: et rl}
\end{align}
When policies are the same, $\forall k,t,s$, $\etak_t(s,a) = 1$.  Define $\underline{\eta}_t \doteq \min_{k,s,a} \etak_t(s,a)$ and  $\overline{\eta} \doteq \max_{k,a} \etak_t(s,a)$, we have $\forall t,k,s, a$, 
\begin{align}\label{eq: eta inequality rl}
\textstyle\underline{\eta}_t \leq \etak_t(s,a) \leq \overline{\eta}_t.
\end{align}
Next, to extend the variance reduction property from statistics (Lemma~\ref{lemma: better than each average}) into reinforcement learning, we also allow each target policy to generate $n_k$ samples. With a similar notation, we have the empirical average for all $\pik$ as
\begin{align}
\Eonk_{t: T-1} \doteq \textstyle \frac{\sum_{i=1}^
{n_k}  \pdisgpik\qty(\tau^{[\pik_{t:T-1},i]}_{t:T-1})}{n_k}, \label{def: Eon rl}
\end{align}
where $\tau^{[\pik,i]}$ is the $i$th trajectory obtained by running $\pik$.
To achieve a fair comparison, when doing off-policy estimation by following $\hat \mu$, we generate $n=\sum_{k=1}^Kn_k$ samples. Likewise, define
\begin{align}
\textstyle\Eoffk_{t: T-1}  \doteq \frac{\sum_{i=1}^n  \pdisgpik\qty(\tau^{[\hat\mu_{t:T-1},i]}_{t:T-1})}{n}. \label{def: Eoff rl}
\end{align}
We have the following theorem.
\begin{restatable}[Variance Reduction with Same Sample Sizes]{theorem}{reOOtheoremOObetterOOthanOOeachOOaverageOOrl}
\label{theorem: better than each average rl}
$\forall k $, $\forall t$, $\forall s$, 
\begin{align}
&  \textstyle\V\left(E^{\text{off},\pik}_{t:T-1}\mid S_t=s\right) \leq \V\left(E^{\text{on,}\pik}_{t:T-1}\mid S_t=s\right).
\end{align}
if the similarity $\eta$ has $\forall k, \forall t,\forall s$,
\begin{align}
&\textstyle\sqrt{\frac{\overline{\eta_t}}{\underline{\eta_t}}} \qty(\sum_{a} \pik_t(a|s)\sqrt{\hat q_{\pik,t}(a|s)})^2-\qty(1-\frac{n_k}{n})\Delta^{(k)}_t(s)\\
\leq&\textstyle\sum_{a} \pik_t(a|s) \hat q_{\pik,t}(s,a), \label{eq: eta sufficient average rl}
\end{align}
where 
\begin{align}\label{eq: def Delta k rl}
\textstyle\Delta^{(k)}_t(s) \doteq\E_{A_t\sim \hat\mu_t}\left[{\rho^{\pik, \hat\mu}}^2 \nu_{\pik,t}(S_t, A_t) \mid S_t=s\right] \\
\textstyle+ \V_{A_t \sim \hat\mu_t}\qty({\rho^{\pik, \hat\mu}} q_{\pik, t}(S_t, A_t)\mid S_t=s).
\end{align}
\end{restatable}
Its proof is in the appendix.
We then compare the datasets when the behavior policy $\hat \mu$ and each target policy $\pik$ both generate $n$ samples, resulting in a “$n$-to-$nK$" comparison, 
similar to Lemma~\ref{lemma: better than each}.
\begin{restatable}[Variance Reduction]{theorem}{reOOtheoremOObetterOOthanOOeachOOrl}
\label{theorem: better than each rl}
$\forall k$, $\forall t$, $\forall s$,
\begin{align}
&\textstyle\V\left(\pdisgpik\qty(\tau^{ \hat \mu_{t:T-1}}_{t:T-1})\mid S_t = s\right) \\ \leq &\textstyle\V\left(\pdisgpik\qty(\tau^{\pik_{t:T-1}}_{t:T-1})\mid S_t = s\right),
\end{align}
if the similarity $\eta$ has $\forall k, \forall t, \forall s,$
\begin{align}\label{eq: eta sufficient each rl}
&\textstyle\sqrt{\frac{\overline{\eta}_t}{\underline{\eta}_t}} \qty(\sum_{a} \pik_t(a|s)\sqrt{\hat q_{\pik,t}(s,a)})^2 \\
 \leq &\textstyle\sum_{a}\pik_t(a|s)\hat q_{\pik, t}(s, a) .
\end{align}    
\end{restatable}
Its proof is in the appendix. This theorem implies that in the multi-step RL setting, running our tailored behavior policy $\hat \mu$ also ensures that the number of required samples does not scale with the number of target policies under similarity conditions. 
The reduced variance of our method depends on the similarity between target policies, which can be easily checked through learning $\hat q$ with offline data. Thus, if RL practitioners are not confident in the similarity between target policies, they can verify it before actual deployment without consuming any online data.

\begin{algorithm}[t]
\caption{Multi-Policy Evaluation (MPE) algorithm}
\label{alg: ODI algorithm}
\begin{algorithmic}[1]
\STATE {\bfseries Input:} 
$K$ target policies $\pik$,\\
an offline dataset $\mathcal{D} = \qty{(t_i,s_i,a_i,r_i,s'_i)}_{i=1}^m$
\STATE {\bfseries Output:} a behavior policy $\hat{\mu}$
\STATE Approximate $q_{\pik,t}$ from $\mathcal{D}$ using any offline RL method (e.g. Fitted Q-Evaluation)
\STATE Compute $\hat{r}_{\pik,i}$  for data pairs in $\mathcal{D}$ by \eqref{def: hat r}
\STATE Construct $\mathcal{D}^{(k)} \doteq \qty{(t_i,s_i,a_i,\hat{r}_{\pik,i},s'_i)}_{i=1}^m$ 
\STATE Approximate $\hat{q}_{\pik,t}$ from  $\mathcal{D}^{(k)} $ by \eqref{eq: recursive hat q} using any offline method (e.g. Fitted Q-Evaluation)
\STATE \textbf{Return:} $\textstyle \hat \mu_t(a|s) \propto 
\sqrt{\sum_{k=1}^K {\pik_t}(a|s)^2 \hat q_{\pik, t}(s, a)}$\\
\end{algorithmic}
\end{algorithm}
\begin{figure}[H]
\includegraphics[width=0.45\textwidth]{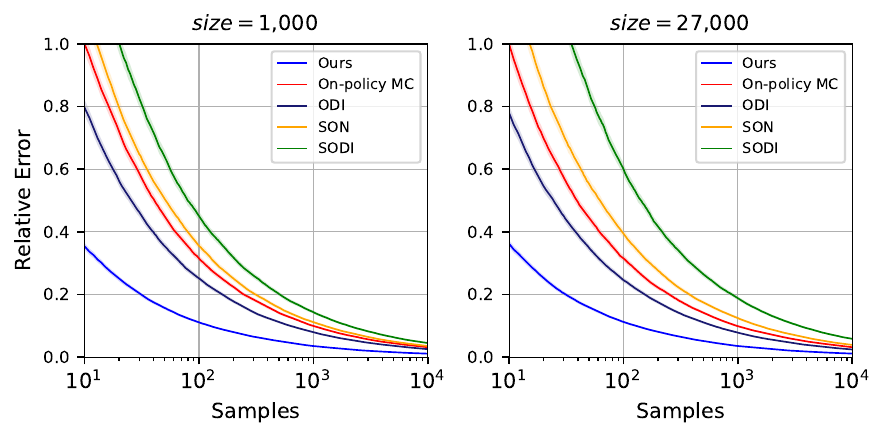}
\centering
\caption{Results on Gridworld. 
Each curve is averaged over 900 runs (30 groups of policies, each having 30 independent runs).
Shaded regions denote standard errors and are invisible for some curves because they are too small.
}
\label{fig:gridworld}
\end{figure}

\begin{figure*}[t]
\includegraphics[width=0.9\textwidth]{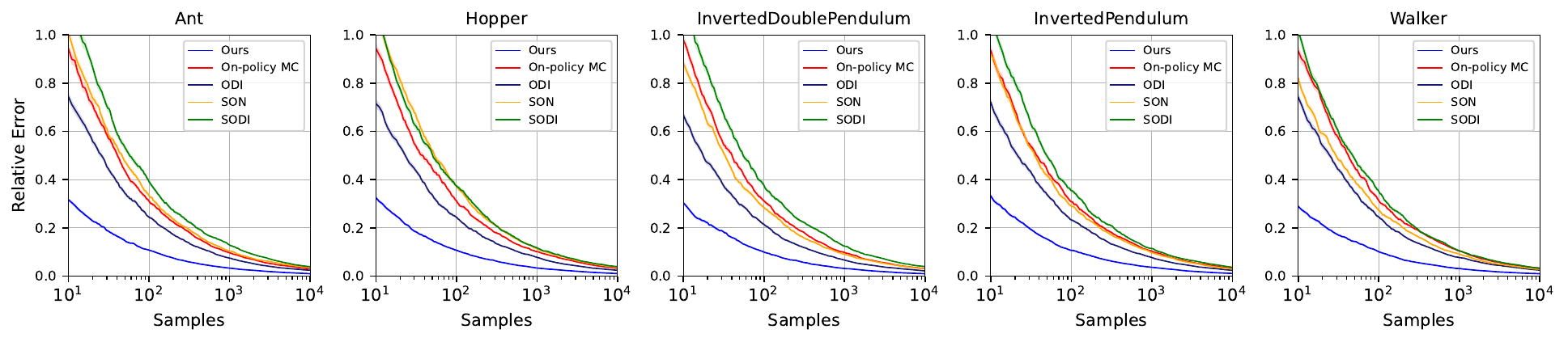}
\centering
\caption{
Results on MuJoCo. 
Each curve is averaged over 900 runs (30 groups of target policies, each having 30 independent runs). 
Shaded regions denote standard errors and are invisible for some curves because they are too small.
}
\label{fig:mujoco}
\end{figure*}
\section{Empirical Results}
We evaluate $K=10$ target policies simultaneously by executing the tailored behavior policy $\hat\mu$ with $n$ total samples. 
We name our method multiple policy evaluation (MPE) estimator. We present our empirical comparisons with the following baselines:
\tb{(1)} The canonical on-policy Monte Carlo estimator with $n_k$ samples for each target policy $\pik$, summing to a total of  $\textstyle n = \sum_{k=1}^{K} n_k$ samples.
\tb{(2)} The offline data informed estimator (ODI, \citet{liu2024efficient}) that runs each behavior policy (designed for each target policy $\pik$) for $n_k$ samples,  summing to a total of  $\textstyle n = \sum_{k=1}^{K} n_k$ samples.
\tb{(3)} The shared-sample on-policy Monte Carlo estimator (SON), where we evaluate each target policy with shared data collected by canonical on-policy Monte Carlo estimators of all $K$ policies, resulting in $n = \sum_{k=1}^{K} n_k$ samples used to evaluate every target policy. 
\tb{(4)} The shared-sample ODI estimator (SODI), where we evaluate each target policy with shared data collected by ODI estimators of all $K$ policies. Since each single behavior policy from the ODI estimator collects $n_k$ samples, each target policy in SODI leverages $n = \sum_{k=1}^{K} n_k$ samples. 

As a demonstration of concept, we set $K=10$ and $n_k=\frac{n}{K}$ for each of the $10$ target policies.  
Target policies are drawn from the training process of proximal policy optimization (PPO) algorithm \citep{schulman2017proximal}. 
We learn our behavior policy $\hat \mu$ using Algorithm~\ref{alg: ODI algorithm}. 
Hyperparameters are the same across all MuJoCo and Gridworld experiments.
Experimental details are in the appendix.

\textbf{Gridworld:} We use Gridworld with $m^3=1,000$ and $m^3=27,000$ states, 
where each Gridworld has a width $m$ and height $m$ with a 
 time horizon $T=m$. 

\begin{table}[H]
    \centering
\begin{tabular}{llllll}
\toprule
Env & Ours & On-policy & ODI & SON & SODI \\
 Size & & MC &  & &  \\
\midrule
1,000 & \textbf{0.125} & 1.000 & 0.637 & 1.289 & 2.073 \\
27,000 & \textbf{0.129} & 1.000 & 0.601 & 1.561 & 3.532 \\
\bottomrule
\end{tabular}
\caption{Relative variance of estimators on Gridworld. The relative variance is defined as the variance of each estimator divided by the variance of the on-policy Monte Carlo estimator. Numbers are averaged over 900 independent runs (30 groups of target policies, each having 30 independent runs).}
\label{table: gridworld variance ratio}
\end{table}
\begin{table}[H]
    \centering
\begin{tabular}{llllll}
\toprule
Env & Ours & On-policy & ODI & SON & SODI \\
 Size & & MC &  & &  \\
\midrule
1,000 & \textbf{126} & 1000 & 632 & 1264 & 2046 \\
27,000 & \textbf{131} & 1000 & 629 & 1568 & 3501 \\
\bottomrule
\end{tabular}
\caption{Episodes needed to achieve the same of estimation accuracy that on-policy Monte Carlo achieves with $1000$ episodes. Numbers are averaged over 900 independent runs (30 groups of target policies, each having 30 independent runs) and their standard errors are shown in Figure \ref{fig:gridworld}.}
\label{table: gridworld samples}
\end{table}

Figure~\ref{fig:gridworld} shows our method outperforms all baselines by a large margin. 
The \textit{relative error} is defined as the estimation error divided by the estimation error of the on-policy MC at the beginning of x-axis. The \textit{samples} on the x-axis represents the total online episodes for multi-policy evaluation. 
The blue line in the graph is below other lines, indicating that our method requires fewer samples to achieve the same accuracy. To quantify the variance reduction, Table~\ref{table: gridworld variance ratio} shows our method reduces variance to about $12.5\%$ compared with the on-policy Monte Carlo estimator. Table~\ref{table: gridworld samples} shows that to achieve the same estimation error that the on-policy Monte Carlo estimator achieves with $1000$ samples, our estimator only needs about $130$ samples saving about $87\%$ of online interactions, achieving state-of-the-art performance.

\textbf{MuJoCo:}
Next, we conduct experiments in MuJoCo robot simulation tasks \citep{todorov2012mujoco}.
MuJoCo is a physics engine containing various stochastic environments, where the goal is to control a robot to achieve different behaviors such as walking, jumping, and balancing. 
Figure~\ref{fig:mujoco} shows our method is consistently better than all baselines. The tables in the appendix show similar patterns as in the Gridworld experiment. In particular, our estimator reduces the variance to about $10\%$ compared with the on-policy Monte Carlo estimator and saves about $90\%$ of online interactions.

An interesting observation to demonstrate the discrepancy among target policies is that SODI and SON  generally perform worse than On-policy MC and ODI.
This result suggests that when target policies lack sufficient similarity, reusing data without a carefully designed joint behavior policy leads to high-variance estimation. Additionally, while ODI outperforms On-policy MC, SODI performs worse than SON. This may be because each behavior policy in SODI is specially tailored for its own target policy, making it vulnerable to target policy change. \textit{These observations confirm the notorious difficulty of data sharing across multiple policies, highlighting the need for a tailored and shared behavior policy to efficiently facilitate data sharing.} 

\section{Conclusion}
In this paper, we introduce a novel approach for multi-policy evaluation by designing a tailored behavior policy that efficiently and unbiasedly evaluates multiple target policies.

Theoretically, 
our method 
eliminates the need for restrictive assumptions or infeasible knowledge required by previous methods.
Our method achieves lower variance compared to on-policy evaluation for each target policy under similarity conditions (Theorem~\ref{theorem: better than each average rl}, Theorem~\ref{theorem: better than each rl}) and ensures the number of required samples does not scale with the number of target policies when similarity conditions hold.

Empirically,  our method outperforms previously best-performing methods, achieving state-of-the-art performance across various environments.
One promising future direction is to extend our variance reduction method to policy improvement and achieve efficient policy learning.

\section{Acknowledgements}
This work is supported in part by the US National Science Foundation (NSF) under grants III-2128019 and SLES-2331904. Claire Chen is supported in part by an Ingrassia Family Echols Scholars Research Grant.

\bibliography{bibliography}

\newpage
\appendix
\onecolumn

\section{Proofs}
\subsection{Proof of Lemma~\ref{lemma: stats unbiasedness}}
\label{sec proof lem stats unbiasedness}
\begin{proof}
$\forall k$, 
\begin{align}
  \E_{A\sim\mu}\left[\rhok(A)q(A)\right] =& \sum_{a \in \qty{a|\mu(a) > 0}} \mu(a) \frac{\pik(a)}{\mu(a)} q(a) \\
  =& \sum_{a \in \qty{a|\mu(a) > 0}} \pik(a) q(a) \\
  =& \sum_{a \in \qty{a|\mu(a) > 0}} \pik(a) q(a) + \sum_{a \in \qty{a | \mu(a) = 0}} \pik(a)q(a) \explain{$\mu \in \Lambda$} \\
  =&\sum_a \pik(a)q(a) \\
  =&\E_{A\sim\pik}\left[q(A)\right].
\end{align} 

The intuition in the third equation is that the sample $a$ where $\mu$ does not cover $\pik$ must satisfy $q(a) = 0$,
i.e.,
this sample does not contribute to the expectation anyway.
\end{proof}
\subsection{Proof of Lemma \ref{lemma: math optimal}}\label{append:math optimal}

\begin{proof}\hspace{1cm}\\
Define
\begin{align}\label{def: A +}
\fA_+ \doteq \qty{a \mid  \exists k, \pik(a)q(a) \neq 0}.
\end{align} 
For any $\mu \in \Lambda$, 
we expand the variance in \eqref{eq: math-optimization} as 
\begin{align}
&\sum_{k \in [K]} \V_{A\sim \mu}(\rhok(A)q(A)) \label{eq: math-variance} \\
=& \sum_{k \in [K]} \E_{A\sim \mu}[(\rhok(A)q(A))^2] - \E_{A\sim \mu}[\rhok(A)q(A)]^2 \\
=& \sum_{k \in [K]} \E_{A\sim \mu}[(\rhok(A)q(A))^2] - \E_{A\sim \pik}[q(A)]  ^2\explain{Lemma~\ref{lemma: stats unbiasedness}} \\
=& \sum_{k \in [K]} \sum_{a \in \qty{a \mid \mu(a) > 0}} \frac{{\pik}(a)^2q(a)^2}{\mu(a)}- \E_{A\sim \pik}[q(A)]^2 \\
=& \sum_{k \in [K]} \sum_{a \in \qty{a \mid \mu(a) > 0} \cap \fA_+} \frac{{\pik}(a)^2q(a)^2 }{\mu(a)}- \E_{A\sim \pik}[q(A)]^2 \explain{$\forall a \notin \fA_+, \forall k, \pik(a)q(a) = 0$} \\ 
=&  \sum_{a \in \fA_+} \frac{ \sum_{k \in [K]} {\pik}(a)^2q(a)^2 }{\mu(a)}- \sum_{k \in [K]} \E_{A\sim \pik}[q(A)]^2. \explain{$\mu \in \Lambda$} 
\end{align}
The second term is a constant and is unrelated to $\mu$. 
Solving the optimization problem \eqref{eq: math-optimization} is,
therefore, equivalent to solving
\begin{align} 
\text{min}_{\mu \in \Lambda}  \quad & 
\sum_{a \in \fA_+} \frac{\sum_{k \in [K]} {\pik}(a)^2q(a)^2 }{\mu(a)} \label{eq: math-optimization-2}.
\end{align}
\textbf{Case 1: $\abs{\fA_+} = 0$} \\
In this case,
optimization target \eqref{eq: math-variance} is always $0$. Any $\mu \in \Lambda$ is optimal.
In particular, $\mu^*(a) = \frac{1}{\fA}$ is optimal. \\
\textbf{Case 2: $\abs{\fA_+} > 0$} \\ 
The definition of $\Lambda$ in~\eqref{eq: stats search space} can be equivalently expressed, using contraposition, as 
\begin{align}
  \Lambda = \qty{\mu \in \Delta(\fA) \mid \forall a, a \in \fA_+ \implies \mu(a) > 0}.
\end{align}
The optimization problem~\eqref{eq: math-optimization-2} can then be equivalently written as
\begin{align}
 \text{min}_{\mu \in \Delta(\fA)}  \quad & 
\sum_{a \in \fA_+} \frac{\sum_{k \in [K]} 
 {\pik}(a)^2q(a)^2 }{\mu(a)} \label{eq: math-optimization-3} \\
\text{s.t.} \quad &
\mu(a) > 0 \quad \forall a \in \fA_+.
\end{align}
If for some $\mu$ we have
$\sum_{a\in \fA_+} \mu(a) < 1$,
then there must exist some $a_0 \notin \fA_+$ such that $\mu(a_0) > 0$.
By the definition of $\fA^+$ \eqref{def: A +},  $\forall a_0 \notin \fA_+$, 
\begin{align}
\sum_{k \in [K]} {\pik}^2(a_0)q^2(a_0) = 0.
\end{align}
This means $a_0$ 
does not contribute to the summation  in the objective function of~\eqref{eq: math-optimization-3},
we can move the probability mass on $a_0$ to some other $a_1 \in \fA_+$ to increase $\mu(a_1)$ to further decrease the objective.
In other words,
any optimal solution $\mu$ to~\eqref{eq: math-optimization-3} must put all its probability mass on $\fA_+$.
This motivates the following problem
\begin{align} 
\text{min}_{z \in \Delta(\fA_+)}  \quad & 
\sum_{a \in \fA_+} \frac{\sum_{k \in [K]}{\pik}(a)^2q(a)^2}{z(a)} \label{eq: math-optimization-4} \\
\text{s.t.} \quad &
z(a) > 0 \quad \forall a \in \fA_+.
\end{align}
In particular, if $z^*$ is an optimal solution to~\eqref{eq: math-optimization-4},
then an optimal solution to~\eqref{eq: math-optimization-3} can be constructed as
\begin{align}
  \label{eq: opt mu construnction}
  \mu^*(a) = \begin{cases}
    z^*(a) & a \in \fA_+, \\
    0 & \text{otherwise}.
  \end{cases}
\end{align}
Let $\R_{++} \doteq (0, +\infty)$.
According to the Cauchy-Schwarz inequality,
for any $z \in \R_{++}^\abs{\fA_+}$, 
we have
\begin{align}
\left(\sum_{a \in \fA_+} \frac{\sum_{k \in [K]} {\pik}(a)^2q(a)^2}{z(a)}\right)\left(\sum_{a \in \fA_+} z(a)\right) \geq&
\left(\sum_{a\in\fA_+} \frac{\sqrt{\sum_{k \in [K]} {\pik}(a)^2 q(a)^2}}{\sqrt{z(a)}} \sqrt{z(a)}\right)^2 \\
=& \left(\sum_{a\in\fA_+} \sqrt{\sum_{k \in [K]}{\pik}(a)^2q(a)^2} \right)^2.
\end{align}
It can be easily verified that the equality holds for 
\begin{align}
z^*(a) \doteq \frac{ \sqrt{\sum_{k \in [K]}{\pik}(a)^2q(a)^2}  }{\sum_{b} \sqrt{\sum_{k \in [K]}{\pik}(b)^2q(b)^2} } > 0.
\end{align}
Since $\sum_{a\in\fA_+} z^*(a) = 1$,
we conclude that $z^*$ is an optimal solution to~\eqref{eq: math-optimization-4}.
An optimal solution $\mu^*$ to~\eqref{eq: math-optimization} can then be constructed according to~\eqref{eq: opt mu construnction}.
Making use of the fact that $\forall a \notin \fA, \forall k, \pik(a)q(a) = 0$,
this $\mu^*$ can be equivalently expressed as 
\begin{align}
\mu^*(a) = \frac{\sqrt{\sum_{k \in [K]}{\pik}(a)^2q(a)^2} }{\sum_{b \in \fA} \sqrt{\sum_{k \in [K]}{\pik}(b)^2q(b)^2}},
\end{align}
which completes the proof.
\end{proof}

\subsection{Proof of Lemma~\ref{lemma: better than each average}}
\label{appendix: better than each average}

\begin{proof}
$\forall k$, we first derive an upper-bound on $\V_{A \sim \mus}\qty(\rhoks(A)q(A))$,
\begin{align}
&\V_{A \sim \mus}(\rhoks(A)q(A)) \\
=& \E_{A \sim \mus}\qty[ (\rhoks(A)q(A))^2 ] - \E_{A \sim \mus}\qty[\rhoks(A)q(A)]^2 \\ 
=& \E_{A \sim \mus}\qty[ (\rhoks(A)q(A))^2 ] -  \E_{A \sim \pik}\qty[q(A)]^2   \explain{Lemma \ref{lemma: stats unbiasedness}}\\ 
=& \E_{A \sim \mus}\qty[ \frac{\wk(a)}{\mus(a)^2} ] -  \E_{A \sim \pik}\qty[q(A)]^2   \explain{By \eqref{def: w}}\\ 
=& \sum_{a} {\wk}(a) \frac{1}{\mus(a)} -  \E_{A \sim \pik}\qty[q(A)]^2 \\
=& \sum_{a} {\wk}(a) \qty(\frac{\sum_{b} \sqrt{\sum_{j \in [K]}{\wj}(b)}} {\sqrt{\sum_{j \in [K]}\wj(a))}}) -  \E_{A \sim \pik}\qty[q(A)]^2  \explain{By \eqref{def: w} and definition of $\mu^*$}\\
=& \sum_{a} {\wk}(a) \qty(\frac{\sum_{b} \sqrt{K\barw(b)}} {\sqrt{K\barw(a)}}) -  \E_{A \sim \pik}\qty[q(A)]^2 \explain{By \eqref{def: barw}}\\ 
=& \sum_{a} {\wk}(a) \qty(\frac{\sum_{b} \sqrt{\barw(b)}} {\sqrt{\barw(a)}}) -  \E_{A \sim \pik}\qty[q(A)]^2 \\    
=& \sum_{a} {\wk}(a) \qty(\frac{\sum_{b} \sqrt{\frac{\wk(b)}{\etak(b)}}} {\sqrt{\frac{\wk(a)}{\etak(a)}}}) -  \E_{A \sim \pik}\qty[q(A)]^2 \explain{By \eqref{def: eta}} \\   
\leq& \sum_{a} {\wk}(a) \qty(\frac{\sum_{b} \sqrt{\frac{\wk(b)}{\underline{\eta}}}} {\sqrt{\frac{\wk(a)}{\overline{\eta}}}}) -  \E_{A \sim \pik}\qty[q(A)]^2 \explain{By \eqref{eq: eta inequality}} \\   
=& \sum_{a} {\wk}(a) \qty(\frac{\sqrt{\frac{1}{\underline{\eta}}}\sum_{b} \sqrt{\wk(b)}} {\sqrt{\frac{1}{\overline{\eta}}}\sqrt{\wk(a)}}) -  \E_{A \sim \pik}\qty[q(A)]^2  \\   
= & \sum_{a} {\wk}(a) \qty(\frac{\sum_{b} \sqrt{\overline{\eta}\wk(b)}} {\sqrt{\underline{\eta}\wk(a)}}) -  \E_{A \sim \pik}\qty[q(A)]^2 \\   
=& \sqrt{\frac{\overline{\eta}}{\underline{\eta}}}\sum_{a} {\wk}(a) \qty(\frac{\sum_{b} \sqrt{\wk(b)}} {\sqrt{\wk(a)}}) -  \E_{A \sim \pik}\qty[q(A)]^2  \\   
=& \sqrt{\frac{\overline{\eta}}{\underline{\eta}}} \qty(\sum_{a} \sqrt{{\wk}(a)}) \qty(\sum_{b} \sqrt{\wk(b)} ) -  \E_{A \sim \pik}\qty[q(A)]^2  \\   
=& \sqrt{\frac{\overline{\eta}}{\underline{\eta}}} \qty(\sum_{a} \sqrt{{\wk}(a)})^2 -  \E_{A \sim \pik}\qty[q(A)]^2  \\ 
=&\explaind{\sqrt{\frac{\overline{\eta}}{\underline{\eta}}} \qty(\sum_{a} \pik(a) q(a))^2 -  \E_{A \sim \pik}\qty[q(A)]^2.}{By \eqref{def: w}}\label{eq: variance upper bound of mu *}
\end{align}
Then, $\forall k \in [K]$, observe the following inequality,
\begin{align}
 &\frac{1}{n} \qty[\sqrt{\frac{\overline{\eta}}{\underline{\eta}}} \qty(\sum_{a} \pik(a) q(a))^2 -  \E_{A \sim \pik}\qty[q(A)]^2]\\
 =&\frac{1}{n} \left[\sqrt{\frac{\overline{\eta}}{\underline{\eta}}} \qty(\sum_{a} \pik(a) q(a))^2-\qty(\frac{n}{n_k} -1) \qty[ \sum_{a} \pik(a) q(a)^2 - \qty(\sum_{a} \pik(a) q(a))^2]\right.\\
 &\left.+ \qty(\frac{n}{n_k} -1) \qty[ \sum_{a} \pik(a) q(a)^2 - \qty(\sum_{a} \pik(a) q(a))^2]- \E_{A \sim \pik}\qty[q(A)]^2\right]\\
 \leq&\frac{1}{n} \left[\sum_{a} \pik(a) q(a)^2 + \qty(\frac{n}{n_k} -1) \qty[ \sum_{a} \pik(a) q(a)^2 - \qty(\sum_{a} \pik(a) q(a))^2]- \E_{A \sim \pik}\qty[q(a)]^2\right]\explain{By \eqref{eq: eta sufficient average}}\\
 =&\frac{1}{n} \left[\frac{n}{n_k}   \sum_{a} \pik(a) q(a)^2 - \qty(\frac{n}{n_k} -1) \E_{A \sim \pik}\qty[q(a)]^2  - \E_{A \sim \pik}\qty[q(a)]^2\right]\\
 =&\frac{1}{n} \left[\frac{n}{n_k}   \sum_{a} \pik(a) q(a)^2 - \frac{n}{n_k} \E_{A \sim \pik}\qty[q(a)]^2\right]\\
=& \frac{1}{n_k} \qty[ \sum_{a} \pik(a) q(a)^2 -  \E_{A \sim \pik}\qty[q(a)]^2].  \label{eq: average sufficient used form}
\end{align}

Now, we have $\forall k \in [K]$,
\begin{align}
&\V_{A \sim \mus}\qty(\Eoffk) \\
=& \V_{A \sim \mus}\qty(\frac{\sum_{i=1}^n  \rhoks(\Asi)q(\Asi)}{n})   \explain{By \eqref{def: Eoffk}}   \\
=& \frac{1}{n^2} \V_{A \sim \mus}\qty(\sum_{i=1}^n  \rhoks(\Asi)q(\Asi))     \\
=& \frac{1}{n} \V_{A \sim \mus}\qty( \rhoks(A)q(A))  \explain{Independence of $\Asi$}   \\
\leq& \frac{1}{n} \qty[\sqrt{\frac{\overline{\eta}}{\underline{\eta}}} \qty(\sum_{a} \pik(a) q(a))^2 -  \E_{A \sim \pik}\qty[q(A)]^2] \explain{by \eqref{eq: variance upper bound of mu *}} \\
\leq& \frac{1}{n_k} \qty[ \sum_{a} \pik(a) q(a)^2 -  \E_{A \sim \pik}\qty[q(A)]^2] \explain{by \eqref{eq: average sufficient used form}}\\
=& \frac{1}{n_k }\V_{A \sim \pik}\qty( q(A) ) \\
=& \frac{1}{n_k^2 }\V_{A \sim \pik}\qty( \sum_{i=1}^{n_k}q(\Aki) ) \explain{Independence of $\Aki$}   \\
=& \V_{A \sim \pik}\qty(\frac{ \sum_{i=1}^{n_k}q(\Aki)}{ n_k} ) \\
=& \V_{A \sim \pik}\qty(\Eonk).\explain{By \eqref{def: Eonk}}
\end{align}

\end{proof}

\subsection{Proof of Lemma~\ref{lemma: better than each}}
\label{appendix: better than each}
\begin{proof}

When sampling from the target policy $\pik$, we have
$\forall k$, 
\begin{align}
&\V_{A \sim \pik} (q(A)) \\
=&\E_{A \sim \pik} \qty[q(A)^2] - \E_{A \sim \pik} \qty[q(A)] ^2 \\
=&\sum_{a} \pik(a) q(a)^2  - \E_{A \sim \pik} \qty[q(A)] ^2.\label{eq: onpolicy variance stats}
\end{align}

With the sufficient condition \eqref{eq: eta sufficient}, we show the variance reduction. $\forall k$,
\begin{align}
&\V_{A \sim \mus}\qty(\rhoks(A)q(A)) \\
=& \sqrt{\frac{\overline{\eta}}{\underline{\eta}}} \qty(\sum_{a} \pik(a)q(a))^2 -  \E_{A \sim \pik}\qty[q(A)]^2  \explain{by \eqref{eq: variance upper bound of mu *}}\\
\leq& \sum_{a} \pik(A) q(a)^2
-  \E_{A \sim \pik}\qty[q(A)]^2  \explain{by \eqref{eq: eta sufficient}} \\
=& \V_{A \sim \pik}\qty(q(A)).\explain{By \eqref{eq: onpolicy variance stats}} 
\end{align}

\end{proof}

\subsection{Proof of Theorem~\ref{lemma: rl pdis unbaised}}
\label{sec lem rl pdis unbaised}
Before proving Theorem~\ref{lemma: rl pdis unbaised}, we first present an auxiliary lemma that is a stronger version of Lemma~\ref{lemma: stats unbiasedness}.

\begin{lemma}\label{lemma: stats unbiasedness stronger}
$\forall \mu \in \Lambda$, $\forall k, \forall t, \forall s$,
\begin{align}
\E_{A_t\sim\mu_t}\left[\rhok_t q_{\pik, t}(S_t, A_t) \mid S_t = s\right] 
=\E_{A_t\sim\pik_t}\left[q_{\pik, t}(S_t, A_t) \mid S_t = s \right].
\end{align}
\end{lemma}

\begin{proof}
$\forall \mu \in \Lambda$, $\forall k, \forall t, \forall s$,
\begin{align}
&\E_{A_t\sim\mu_t}\left[\rhok_t q_{\pik, t}(S_t, A_t) \mid S_t = s\right] \\
=& \sum_{a \in \qty{a|\mu_t(a | s) > 0}} \mu_t(a | s) \frac{\pik_t(a | s)}{\mu_t(a | s)} q_{\pik, t}(s, a) \\
=& \sum_{a \in \qty{a|\mu_t(a | s) > 0}} \pik_t(a | s) q_{\pik, t}(s, a) \\
=& \sum_{a \in \qty{a|\mu_t(a | s) > 0}} \pik_t(a | s) q_{\pik, t}(s, a) + \sum_{a \in \qty{a | \mu_t(a | s) = 0}} \pik_t(a | s)q_{\pik, t}(s, a) \explain{$\mu \in \Lambda$} \\
=&\sum_a \pik_t(a | s)q_{\pik, t}(s, a) \\
=&\E_{A_t\sim\pik_t}\left[q_{\pik, t}(S_t, A_t) \mid S_t = s  \right].
\end{align} 

\end{proof}

Now, we are ready to prove 
Theorem~\ref{lemma: rl pdis unbaised}.
\begin{proof}
We proceed via induction.
$\forall k$,   for $t = T-1$,
we have
\begin{align}
&\E\left[\pdisgpik\qty(\tau^{\mu_{t:T-1}}_{t:T-1}) \mid S_t\right] \\
=& \E\left[\rhok_t r(S_t,A_t) \mid S_t\right] \\
=& \E\left[\rhok_t q_{\pik, t}(S_t, A_t) \mid S_t \right] \\
=& \E_{A_t \sim \pik_t(\cdot \mid S_t)}\left[q_{\pik, t}(S_t, A_t) \mid S_t\right] \explain{Lemma~\ref{lemma: stats unbiasedness stronger}} \\
=& v_{\pik, t}(S_t).
\end{align}
For $t \in [T-2]$,
we have
\begin{align}
&\E\left[\pdisgpik\qty(\tau^{\mu_{t:T-1}}_{t:T-1}) \mid S_t\right] \\
=& \E\left[\rhok_t R_{t+1} + \rhok_t\pdisgpik\qty(\tau^{\mu_{t+1:T-1}}_{t+1:T-1}) \mid S_t\right] \\
=& \E\left[\rhok_t R_{t+1} \mid S_t\right] + \E\left[\rhok_t\pdisgpik\qty(\tau^{\mu_{t+1:T-1}}_{t+1:T-1}) \mid S_t\right] \\
\explain{Law of total expectation}
=& \E\left[\rhok_t R_{t+1} \mid S_t\right] + \E_{A_t \sim \mu_t(\cdot \mid S_t), S_{t+1} \sim p(\cdot \mid S_t, A_t)}\left[ \E\left[\rhok_t\pdisgpik\qty(\tau^{\mu_{t+1:T-1}}_{t+1:T-1}) \mid S_t, A_t, S_{t+1}\right] \mid S_t \right] \\
\explain{Conditional independence and Markov property}
=& \E\left[\rhok_t R_{t+1} \mid S_t\right] + \E_{A_t \sim \mu_t(\cdot \mid S_t), S_{t+1} \sim p(\cdot \mid S_t, A_t)}\left[ \rhok_t \E\left[\pdisgpik\qty(\tau^{\mu_{t+1:T-1}}_{t+1:T-1}) \mid S_{t+1}\right] \mid S_t \right] \\
=& \E\left[\rhok_t R_{t+1} \mid S_t\right] + \E_{A_t \sim \mu_t(\cdot \mid S_t), S_{t+1} \sim p(\cdot \mid S_t, A_t)}\left[ \rhok_t v_{\pik, t+1}(S_{t+1}) \mid S_t \right]
\explain{Inductive hypothesis} \\
=& \E_{A_t \sim \mu_t(\cdot \mid S_t)}\left[\rhok_t q_{\pik, t}(S_t, A_t) \mid S_t\right] \explain{Definition of $q_{\pik, t}$} \\
=& \E_{A_t \sim \pik_t(\cdot \mid S_t)}\left[q_{\pik, t}(S_t, A_t) \mid S_t\right] \explain{Lemma~\ref{lemma: stats unbiasedness stronger}} \\
=& v_{\pik, t}(S_t).
\end{align}
This completes the proof.
\end{proof}

\subsection{Proof of Theorem \ref{lemma: rl-optimal}}\label{append:rl-optima}
To prove Theorem~\ref{lemma: rl-optimal},
we rely on a recursive expression of the PDIS Monte Carlo estimator, which is restated from
\citet{liu2024efficient}, as summarized by the following lemma.
\begin{lemma}[Recursive Expression of Variance] \label{lemma: recursive-var}
For any $\mu \in \Lambda$, $\forall k$, we have for $t = T-1$,
  \begin{align}
    \V\left(\pdisgpik\qty(\tau^{\mu_{t:T-1}}_{t:T-1})\mid S_t\right) = \E_{A_t \sim \mu_t}\left[{\rhok_t}^2 q_{\pik, t}(S_t, A_t)^2 \mid S_t\right] - { v_{\pik, t}(S_t)}^2;
  \end{align}
For $t \in [T-2]$,
\begin{align}
&\V\left(\pdisgpik\qty(\tau^{\mu_{t:T-1}}_{t:T-1})\mid S_t\right) \\
=& \E_{A_t\sim \mu_t}\left[{\rhok_t}^2 \left(\E_{S_{t+1}}\left[\V\left(\pdisgpik\qty(\tau^{\mu_{t+1:T-1}}_{t+1:T-1})\mid S_t\right) \mid S_t, A_t\right] + \nu_{\pik,t}(S_t, A_t) + q_{\pik, t}(S_t, A_t)^2\right) \mid S_t\right] \\
&- v_{\pik, t}(S_t)^2.
\end{align}
\end{lemma}
  \begin{proof}
  When $t\in [T-2]$, we have
  \begin{align}
    \label{eq: tmp4}
  &\V\left(\pdisgpik\qty(\tau^{\mu_{t:T-1}}_{t:T-1})\mid S_t\right) \\
  =& \E_{A_t}\left[ \V\left(\pdisgpik\qty(\tau^{\mu_{t:T-1}}_{t:T-1})\mid S_t, A_t\right)\mid S_t\right] + \V_{A_t}\left(\E\left[\pdisgpik\qty(\tau^{\mu_{t:T-1}}_{t:T-1}) \mid S_t, A_t\right]\mid S_t\right) 
  \explain{Law of total variance} \\
  =& \E_{A_t}\left[ {\rhok_t}^2 \V\left(r(S_t,A_t) + \pdisgpik\qty(\tau^{\mu_{t+1:T-1}}_{t+1:T-1}) \mid S_t, A_t\right)\mid S_t\right] \\
  &+ \V_{A_t}\left(\rhok_t \E\left[r(S_t,A_t) + \pdisgpik\qty(\tau^{\mu_{t+1:T-1}}_{t+1:T-1}) \mid S_t, A_t\right]\mid S_t\right)  
  \explain{Using \eqref{eq: PDIS-recursive}}
  \\
  =& \E_{A_t}\left[ {\rhok_t}^2 \V\left(\pdisgpik\qty(\tau^{\mu_{t+1:T-1}}_{t+1:T-1}) \mid S_t, A_t\right)\mid S_t\right] + \V_{A_t}\left(\rhok_t q_{\pik, t}(S_t, A_t)\mid S_t\right) \explain{Deterministic reward $r$}.
  \end{align}
  Further decomposing the first term, we have
  \begin{align}
    \label{eq: tmp3}
  &\V\left(\pdisgpik\qty(\tau^{\mu_{t+1:T-1}}_{t+1:T-1}) \mid S_t, A_t\right) \\
  =& \E_{S_{t+1}}\left[\V\left(\pdisgpik\qty(\tau^{\mu_{t+1:T-1}}_{t+1:T-1}) \mid S_t, A_t, S_{t+1}\right) \mid S_t, A_t\right] \\
  &+ \V_{S_{t+1}}\left(\E\left[\pdisgpik\qty(\tau^{\mu_{t+1:T-1}}_{t+1:T-1}) \mid S_t, A_t, S_{t+1}\right]\mid S_t, A_t\right) 
  \explain{Law of total variance}
  \\
  =& \E_{S_{t+1}}\left[\V\left(\pdisgpik\qty(\tau^{\mu_{t+1:T-1}}_{t+1:T-1}) \mid S_{t+1}\right) \mid S_t, A_t\right] + \V_{S_{t+1}}\left(\E\left[\pdisgpik\qty(\tau^{\mu_{t+1:T-1}}_{t+1:T-1}) \mid S_{t+1}\right]\mid S_t, A_t\right) \explain{Markov property} \\
  =& \E_{S_{t+1}}\left[\V\left(\pdisgpik\qty(\tau^{\mu_{t+1:T-1}}_{t+1:T-1}) \mid S_{t+1}\right) \mid S_t, A_t\right] + \V_{S_{t+1}}\left(v_{\pik, t+1}(S_{t+1})\mid S_t, A_t\right). \explain{Theorem~\ref{lemma: rl pdis unbaised}}
  \end{align}
  With $\nu_{\pik, t}$ defined in~\eqref{def:nu},
  plugging~\eqref{eq: tmp3} back to~\eqref{eq: tmp4} yields
  \begin{align}
    &\V\left(\pdisgpik\qty(\tau^{\mu_{t:T-1}}_{t:T-1})\mid S_t\right) \\
    =&\E_{A_t}\left[{\rhok_t}^2 \left(\E_{S_{t+1}}\left[\V\left(\pdisgpik\qty(\tau^{\mu_{t+1:T-1}}_{t+1:T-1}) \mid S_{t+1}\right) \mid S_t, A_t\right] + \nu_t(S_t, A_t)\right) \mid S_t\right] \\
    &+ \V_{A_t}\left(\rhok_t q_{\pik, t}(S_t, A_t)\mid S_t\right) \\
    =&\E_{A_t}\left[{\rhok_t}^2 \left(\E_{S_{t+1}}\left[\V\left(\pdisgpik\qty(\tau^{\mu_{t+1:T-1}}_{t+1:T-1}) \mid S_{t+1}\right) \mid S_t, A_t\right] + \nu_t(S_t, A_t)\right) \mid S_t\right] \\
    &+ \E_{A_t}\left[{\rhok_t}^2 q_{\pik, t}(S_t, A_t)^2\mid S_t\right] - \left(\E_{A_t}\left[\rhok_t q_{\pik, t}(S_t, A_t) \mid S_t\right]\right)^2 \\
    =&\E_{A_t}\left[{\rhok_t}^2 \left(\E_{S_{t+1}}\left[\V\left(\pdisgpik\qty(\tau^{\mu_{t+1:T-1}}_{t+1:T-1}) \mid S_{t+1}\right) \mid S_t, A_t\right] + \nu_t(S_t, A_t)\right) \mid S_t\right] \\
    &+ \E_{A_t}\left[{\rhok_t}^2 q_{\pik, t}(S_t, A_t)^2\mid S_t\right] - { v_{\pik, t}(S_t)}^2. \explain{Lemma~\ref{lemma: stats unbiasedness stronger}}
  \end{align}
  When $t =  T-1$, we have
  \begin{align}
  \V\left(\pdisgpik\qty(\tau^{\mu_{t:T-1}}_{t:T-1})\mid S_t\right) =& \V\left(\rhok_t r(S_t, A_t)\mid S_t\right) \\
  =& \V\left(\rhok_t q_{\pik, t}(S_t, A_t)\mid S_t\right) \\
  =&  \E_{A_t}\left[{\rhok_t}^2 q_{\pik, t}(S_t, A_t)^2 \mid S_t\right] - { v_{\pik, t}(S_t)}^2,
  \end{align}
  which completes the proof.
  \end{proof}
Then, to solve the variance minimization problem, we manipulate the variance expression in \eqref{eq: one step optimization}. For any policy $k$, for any $\mu \in\hat \Lambda$, when $t=T-1$,
\begin{align}
\label{eq: locally optimal target t-1}
 &\sum_{k=1}^K \V\left(\pdisgpik\qty( \tau^\qty{\mu_t,\pik_{t+1},\dots, \pik_{T-1}}_{t:T-1})\mid S_t = s\right)  \\
 =&\sum_{k=1}^K  \E_{A_t \sim \mu_t}\left[{\rhok_t}^2 q_{\pik, t}(S_t, A_t)^2 \mid S_t\right] - { v_{\pik, t}(S_t)}^2\explain{Lemma~\ref{lemma: recursive-var}}\\
=& \sum_{k=1}^K\E_{A_t\sim\mu_t}\left[{\rhok_t}^2 \hat q_{\pik, t}(S_t, A_t) \mid S_t\right] -v_{\pik,t}(S_t)^2\explain{By \eqref{eq: def q hat}}\\
=& \sum_{k=1}^K\V_{A_t\sim\mu_t}\left(\rhok_t \sqrt{\hat q_{\pik, t}(S_t, A_t)} \mid S_t\right) - \sum_{k=1}^K \E_{A_t\sim\mu_t}\left[\rhok_t \sqrt{\hat q_{\pik, t}(S_t, A_t)} \mid S_t\right] ^2\\
&-v_{\pik,t}(S_t)^2\\
=& \sum_{k=1}^K\V_{A_t\sim\mu_t}\left(\rhok_t \sqrt{\hat q_{\pik, t}(S_t, A_t)} \mid S_t\right) - \sum_{k=1}^K \E_{A_t\sim\pik_t}\left[\sqrt{\hat q_{\pik, t}(S_t, A_t)} \mid S_t\right] ^2\\
&-v_{\pik,t}(S_t)^2.\explain{Lemma~\ref{lemma: stats unbiasedness stronger} and $\mu_t \in \hat \Lambda\subseteq \Lambda$}
\end{align}
For $t \in [T-2]$,
\begin{align}
\label{eq: locally optimal target t-2}
 &\sum_{k=1}^K \V\left(\pdisgpik\qty( \tau^\qty{\mu_t,\pik_{t+1},\dots, \pik_{T-1}}_{t:T-1})\mid S_t = s\right)  \\
 =&\sum_{k=1}^K \E_{A_t\sim \mu_t}\left[{\rhok_t}^2 \left(\E_{S_{t+1}}\left[ \V\left(\pdisgpik\qty(\tau^{{\pik}_{t+1:T-1}}_{t+1:T-1}) \mid S_{t+1}\right) \mid S_t, A_t\right]\right.\right.\\
&\left.\left.+ \nu_{\pik,t}(S_t, A_t) + q_{\pik, t}(S_t, A_t)^2 \right) \mid S_t\right]-v_{\pik,t}(S_t)^2\explain{Lemma~\ref{lemma: recursive-var}}\\
=& \sum_{k=1}^K\E_{A_t\sim\mu_t}\left[{\rhok_t}^2 \hat q_{\pik, t}(S_t, A_t) \mid S_t\right] -v_{\pik,t}(S_t)^2\explain{By \eqref{eq: def q hat}}\\
=& \sum_{k=1}^K\V_{A_t\sim\mu_t}\left(\rhok_t \sqrt{\hat q_{\pik, t}(S_t, A_t)} \mid S_t\right) - \sum_{k=1}^K \E_{A_t\sim\mu_t}\left[\rhok_t \sqrt{\hat q_{\pik, t}(S_t, A_t)} \mid S_t\right] ^2\\
&-v_{\pik,t}(S_t)^2\\
=& \sum_{k=1}^K\V_{A_t\sim\mu_t}\left(\rhok_t \sqrt{\hat q_{\pik, t}(S_t, A_t)} \mid S_t\right) - \sum_{k=1}^K \E_{A_t\sim\pik_t}\left[\sqrt{\hat q_{\pik, t}(S_t, A_t)} \mid S_t\right] ^2\\
&-v_{\pik,t}(S_t)^2.\explain{Lemma~\ref{lemma: stats unbiasedness stronger} and $\mu_t \in \hat \Lambda\subseteq \Lambda$}
\end{align}
Since for both \eqref{eq: locally optimal target t-1} and \eqref{eq: locally optimal target t-2}, the second and third terms are unrelated to $\hat\mu$, solving \eqref{eq: one step optimization} is equivalent to solve
\begin{align}
\label{eq: solve-hat-q}
 \min_{\mu_t \in \hat \Lambda} \quad \sum_{k=1}^K\V_{A_t\sim\mu_t}\left(\rhok_t \sqrt{\hat q_{\pik, t}(S_t, A_t)} \mid S_t\right).
\end{align}
Then, with Lemma~\ref{lemma: math optimal}, we conclude that $\hat \mu_t$ as defined in \eqref{def hat mu} is an optimal solution to \eqref{eq: solve-hat-q}, which completes the proof.
\subsection{Proof of Theorem \ref{theorem: better than each average rl}}
\label{append: better than each average rl}

Before proving Theorem~\ref{theorem: better than each average rl}, given the sufficient condition in \eqref{eq: eta sufficient average rl}, we first observe the following equation.
\begin{align}
\label{eq: sufficient delta manipulate}
 &\E_{A_t\sim \hat\mu_t}\left[{\rho^{\pik, \hat\mu}}^2 \nu_{\pik,t}(S_t, A_t) \mid S_t\right]  + \V_{A_t \sim \hat\mu_t}\qty({\rho^{\pik, \hat\mu}} q_{\pik, t}(S_t, A_t)\mid S_t)\\
 =&\E_{A_t\sim \hat\mu_t}\left[{\rho^{\pik, \hat\mu}}^2 \nu_{\pik,t}(S_t, A_t) \mid S_t\right] +\E_{A_t\sim \hat\mu_t}\left[{\rho^{\pik, \hat\mu}}^2 q_{\pik, t}(S_t, A_t)^2 \mid S_t\right] \\
 &- \E_{A_t\sim \hat\mu_t}\left[{\rho^{\pik, \hat\mu}} q_{\pik, t}(S_t, A_t) \mid S_t\right]^2\\
 =&\E_{A_t\sim \hat\mu_t}\left[{\rho^{\pik, \hat\mu}}^2 \nu_{\pik,t}(S_t, A_t) \mid S_t\right]
 +\E_{A_t\sim \hat\mu_t}\left[{\rho^{\pik, \hat\mu}}^2 q_{\pik, t}(S_t, A_t)^2 \mid S_t\right]-v_{\pik, t}(S_t)^2.\explain{Theorem~\ref{lemma: rl pdis unbaised}}
\end{align}
Thus, we obtain $\forall t,s$,
\begin{align}
\label{eq: sufficient delta used}
&\sqrt{\frac{\overline{\eta}_t}{\underline{\eta}_t}} \qty(\sum_{a} \pik_t(a|s)\sqrt{\hat q_{\pik,t}(a|s)})^2 \\
&- \frac{n - n_k}{n} \qty( \E_{A_t\sim \hat\mu_t}\left[{\rho^{\pik, \hat\mu}}^2 \left(\nu_{\pik,t}(s, A_t) + q_{\pik, t}(s, A_t)^2\right) \mid s\right] -{ v_{\pik, t}(s)}^2 )\\ 
=&\sqrt{\frac{\overline{\eta}_t}{\underline{\eta}_t}} \qty(\sum_{a} \pik_t(a|s)\sqrt{\hat q_{\pik,t}(a|s)})^2 \\
&- \frac{n - n_k}{n} \qty( \E_{A_t\sim \hat\mu_t}\left[{\rho^{\pik, \hat\mu}}^2 \nu_{\pik,t}(S_t, A_t) \mid S_t=s\right]  + \V_{A_t \sim \hat\mu_t}\qty({\rho^{\pik, \hat\mu}} q_{\pik, t}(S_t, A_t))\mid S_t=s) \explain{By \eqref{eq: sufficient delta manipulate}}\\
\leq&  \E_{A_t \sim \pik_t}\left[\hat q_{\pik, t}(S_t, A_t) \mid S_t=s\right]. \explain{By \eqref{eq: eta sufficient average rl}}
\end{align}

We also discover the following inequality. $\forall k$, $\forall s$, $\forall t$,
\begin{align}
 &\frac{1}{n} \qty[\sqrt{\frac{\overline{\eta_t}}{\underline{\eta_t}}} \qty(\sum_{a} \pik_t(a|s) \sqrt{\hat q_{\pik,t}(a|s)})^2 - v_{\pik, t}(s)^2]\\
 =&\frac{1}{n} \left[\sqrt{\frac{\overline{\eta_t}}{\underline{\eta_t}}} \qty(\sum_{a} \pik_t(a|s) \sqrt{\hat q_{\pik,t}(a|s)})^2-\qty(\frac{n}{n_k}-1)\qty( \sum_{a} \pik_t(a|s) \hat q_{\pik,t}(s,a) -v_{\pik, t}(s)^2)
 \right.\\
 &\left.+\qty(\frac{n}{n_k}-1)\qty( \sum_{a} \pik_t(a|s) \hat q_{\pik,t}(s,a) -v_{\pik, t}(s)) - v_{\pik, t}(s)^2\right]\\
\leq &\frac{1}{n} \left[\sum_{a} \pik_t(a|s) \hat q_{\pik,t}(s,a)+\qty(\frac{n}{n_k}-1)\qty( \sum_{a} \pik_t(a|s) \hat q_{\pik,t}(s,a) -v_{\pik, t}(s)^2) - v_{\pik, t}(s)^2\right]\explain{By \eqref{eq: sufficient delta used}}\\
=&\frac{1}{n} \left[\frac{n}{n_k}   \sum_{a} \pik_t(a|s) \hat q_{\pik,t}(s,a) - \qty(\frac{n}{n_k} -1){ v_{\pik, t}(s)}^2  - v_{\pik, t}(s)^2\right]\\
 =&\frac{1}{n} \left[\frac{n}{n_k}   \sum_{a} \pik_t(a|s) \hat q_{\pik,t}(s,a)- \frac{n}{n_k}{ v_{\pik, t}(s)}^2\right]\\
=& \frac{1}{n_k} \qty[ \sum_{a} \pik_t(a|s) \hat q_{\pik,t}(s,a) - { v_{\pik, t}(s)}^2]  \\
=&\frac{1}{n_k }\qty(\E_{A_t \sim \pik_t}\left[\hat q_{\pik, t}(S_t, A_t) \mid S_t=s\right] - { v_{\pik, t}(s)}^2).\label{eq: average sufficient used form rl}
\end{align}

Next, to prove Theorem \ref{theorem: better than each rl}, we present a closed-form representation of the variance of the on-policy estimator.
\begin{lemma}
\label{lemma: recursive variance  
onpolicy rl}
For any $k$, $t$ and $s$,
\begin{align}
&\V\left(\pdisgpik\qty(\tau^{\pik_{t:T-1}}_{t:T-1})\mid S_t=s \right) =\E_{A_t \sim \pik_t}\left[\hat q_{\pik, t}(S_t, A_t) \mid S_t=s\right] - { v_{\pik, t}(s)}^2.
\end{align}
\end{lemma}
\begin{proof}
We proceed via induction. When $t=T-1$, 
\begin{align}
&\V\left(\pdisgpik\qty(\tau^{\pik_{t:T-1}}_{t:T-1})\mid S_t \right)\\
=&\E_{A_t \sim \pik_t}\left[q_{\pik, t}(S_t, A_t)^2 \mid S_t\right] - { v_{\pik, t}(S_t)}^2
\explain{Lemma~\ref{lemma: recursive-var} and on-policy}\\
=& \E_{A_t \sim \pik_t}\left[\hat q_{\pik, t}(S_t, A_t) \mid S_t\right] - { v_{\pik, t}(S_t)}^2.\explain{By \eqref{eq: def q hat}}
\end{align}
When $t\in[T-2]$, 
\begin{align}
&\V\left(\pdisgpik\qty(\tau^{\pi_{t:T-1}}_{t:T-1})\mid S_t\right) \\
    =& \E_{A_t\sim \pi_t}\left[\left(\E_{S_{t+1}}\left[\V\left(\pdisgpik\qty(\tau^{\pi_{t+1:T-1}}_{t+1:T-1})\mid S_t\right) \mid S_t, A_t\right] + \nu_{\pik,t}(S_t, A_t) + q_{\pik, t}(S_t, A_t)^2\right) \mid S_t\right] \\
    &-{ v_{\pik, t}(S_t)}^2 \explain{Lemma~\ref{lemma: recursive-var} and on-policy}\\ 
    =& \E_{A_t \sim \pik_t}\left[\hat q_{\pik, t}(S_t, A_t) \mid S_t\right] - { v_{\pik, t}(S_t)}^2,\explain{By \eqref{eq: def q hat}}
\end{align}
which completes the induction.
\end{proof}

Additionally, we discover the following inequality. $\forall k$, $\forall s$, $\forall t$,

\begin{align}
&\E_{A_t \sim \mu_t}\left[{\rhok_t}^2 
 \hat q_{\pik, t}(S_t, A_t) \mid S_t=s\right] \\
 =&\E_{A_t \sim \hat \mu_t}\left[\frac{\wk_t(S_t, A_t)}{\hat \mu_t(A_t|S_t)^2} \mid S_t=s\right] \explain{By \eqref{def: w rl}}\\
 =&\sum_a \wk_t(s,a)\frac{1}{\hat \mu_t(a|s)}\\
 =& \sum_{a} {\wk_t}(s,a) \qty(\frac{\sum_{b} \sqrt{\sum_{j \in [K]}{\wj_t}(s,b)}} {\sqrt{\sum_{j \in [K]}{\wj_t}(s,a)}})   \explain{By \eqref{def hat mu}}\\
 =& \sum_{a} {\wk_t}(s,a) \qty(\frac{\sum_{b} \sqrt{K\barw_t(s,b)}} {\sqrt{K\barw_t(s,a)}})  \explain{by \eqref{def: barw rl}}\\
 =& \sum_{a} {\wk_t}(s,a) \qty(\frac{\sum_{b} \sqrt{\barw_t(s,b)}} {\sqrt{\barw_t(s,a)}}) \\
=& \sum_{a} {\wk_t}(s,a) \qty(\frac{\sum_{b} \sqrt{\frac{\wk_t(s,b)}{\etak_t(s,b)}}} {\sqrt{\frac{\wk_t(s,a)}{\etak_t(s,a)}}}) \explain{By \eqref{def: et rl}} \\   
\leq& \sum_{a}{\wk_t}(s,a) \qty(\frac{\sum_{b} \sqrt{\frac{\wk_t(s,b)}{\underline{\eta}_t}}} {\sqrt{\frac{\wk_t(s,a)}{\overline{\eta}_t}}})\explain{By \eqref{eq: eta inequality rl}} \\   
=& \sum_{a} {\wk_t}(s, a) \qty(\frac{\sqrt{\frac{1}{\underline{\eta}_t}}\sum_{b} \sqrt{\wk_t(s,b)}} {\sqrt{\frac{1}{\overline{\eta}_t}}\sqrt{\wk_t(s,a)}})\\  
=& \sum_{a} {\wk_t}(s,a) \qty(\frac{\sum_{b} \sqrt{\overline{\eta}_t \wk_t(s,b)}} {\sqrt{\underline{\eta}_t \wk_t(s,a)}}) \\
=& \sqrt{\frac{\overline{\eta}_t}{\underline{\eta}_t}}\sum_{a} {\wk_t}(s,a) \qty(\frac{\sum_{b} \sqrt{\wk_t(s,b)}} {\sqrt{\wk_t(s,a)}})    \\   
=& \sqrt{\frac{\overline{\eta}_t}{\underline{\eta}_t}} \qty(\sum_{a} \sqrt{{\wk_t}(s,a)}) \qty(\sum_{b} \sqrt{\wk_t(s,b)} )    \\
=& \sqrt{\frac{\overline{\eta}_t}{\underline{\eta}_t}} \qty(\sum_{a} \sqrt{{\wk_t}(s,a)})^2   \\
=&\explaind{\sqrt{\frac{\overline{\eta}_t}{\underline{\eta}_t}} \qty(\sum_{a} \pik_t(a|s)\sqrt{\hat q_{\pik,t}(s,a)})^2  .}{By \eqref{def: w rl}}
\label{eq: variance upper bound of mu * rl}
\end{align}

From here, we restate Theorem~\ref{theorem: better than each average rl} and give its proof.
\reOOtheoremOObetterOOthanOOeachOOaverageOOrl*

First, we manipulate the variance of both $E^{\text{off},\pik}_{0:T-1}$ and $E^{\text{on,}\pik}_{0:T-1}$.
\begin{align}
\label{eq: variance of E off}
&\V\left(E^{\text{off},\pik}_{0:T-1} \right) \\
=&\V\left(\frac{\sum_{i=1}^n  \pdisgpik\qty(\tau^{[\hat\mu_{0:T-1},i]}_{0:T-1})}{n}  \right) \explain{By \eqref{def: Eoff rl}} \\
=& \textstyle\frac{1}{n^2} \V\left(\sum_{i=1}^n  \pdisgpik\qty(\tau^{[\hat\mu_{0:T-1},i]}_{0:T-1})   \right)     \\
=& \frac{1}{n}\V\qty(\pdisgpik\qty(\tau^{\hat\mu_{0:T-1}}_{0:T-1})) \explain{Independence of $\tau^{[\hat\mu_{0:T-1},i]}_{0:T-1}$}  \\
=&\frac{1}{n}\E_{S_0}\qty[\V\qty(\pdisgpik\qty(\tau^{\hat\mu_{0:T-1}}_{0:T-1}) \mid S_0 = s) ] + \frac{1}{n}\V_{S_0}\qty(\E\qty[\pdisgpik\qty(\tau^{\hat\mu_{0:T-1}}_{0:T-1}) \mid S_0 = s] )  
\explain{Law of total variance} \\
=&\frac{1}{n}\E_{S_0}\qty[\V\qty(\pdisgpik\qty(\tau^{\hat\mu_{0:T-1}}_{0:T-1}) \mid S_0 = s) ] + \frac{1}{n}\V_{S_0}\qty(v_{\pi,0}(S_0) ).  \explain{Theorem~\ref{lemma: rl pdis unbaised}}
\end{align}
Similarly, we have
\begin{align}
\label{eq: variance of E on}
&\V\left(E^{\text{on},\pik}_{0:T-1} \right) \\
=&\V\qty(\frac{ \sum_{i=1}^{n_k}\pdisgpik\qty(\tau^{[\pik_{0:T-1},i]}_{0:T-1})}{n_k}) \explain{By \eqref{def: Eon rl}}\\
=& \frac{1}{{n_k}^2 }\V\qty( \textstyle\sum_{i=1}^{n_k}\pdisgpik\qty(\tau^{[\pik_{0:T-1},i]}_{0:T-1}) )\\
=& \frac{1}{n_k}\V\qty(\pdisgpik\qty(\tau^{\pik_{0:T-1}}_{0:T-1})) \explain{Independence of $\tau^{[\pik_{0:T-1},i]}_{0:T-1}$}\\
=&\frac{1}{n_k}\E_{S_0}\qty[\V\qty(\pdisgpik\qty(\tau^{\pik_{0:T-1}}_{0:T-1}) \mid S_0 = s) ] + \frac{1}{n_k}\V_{S_0}\qty(\E\qty[\pdisgpik\qty(\tau^{\pik_{0:T-1}}_{0:T-1}) \mid S_0 = s] ) 
\explain{Law of total variance} \\
=&\frac{1}{n_k}\E_{S_0}\qty[\V\qty(\pdisgpik\qty(\tau^{\pik_{0:T-1}}_{0:T-1}) \mid S_0 = s) ] + \frac{1}{n_k} \V_{S_0}\qty(v_{\pi,0}(S_0) ).   \explain{Theorem~\ref{lemma: rl pdis unbaised} and on-policy}
\end{align}

With the manipulated sufficient condition in \eqref{eq: sufficient delta used}, we present the following lemma.
\begin{lemma}
\label{lemma: nk n variance}
    Under the condition in \eqref{eq: eta sufficient average rl}, $\forall k,t,s$,
\begin{align}
&\frac{n_k}{n}\V\left(\pdisgpik\qty(\tau^{\hat\mu_{t:T-1}}_{t:T-1}) \mid S_t=s\right) \\    
\leq &\V\left(\pdisgpik\qty(\tau^{\pik_{t:T-1}}_{t:T-1})\mid S_t=s \right) .
\end{align}
\end{lemma}
\begin{proof}
We proceed via induction.
$\forall k$, $\forall s$, when $t=T-1$,
\begin{align}
&\frac{n_k}{n}\V\left(\pdisgpik\qty(\tau^{\hat\mu_{t:T-1}}_{t:T-1}) \mid S_t=s\right)   \\
 =& \frac{n_k}{n}\qty[\E_{A_t \sim \mu_t}\left[{\rhok_t}^2 q_{\pik, t}(S_t, A_t)^2 \mid S_t=s\right] - { v_{\pik, t}(s)}^2]\explain{Lemma~\ref{lemma: recursive-var}}\\
 =& \frac{n_k}{n}\qty[\E_{A_t \sim \mu_t}\left[{\rhok_t}^2 
 \hat q_{\pik, t}(S_t, A_t) \mid S_t=s\right] - { v_{\pik, t}(s)}^2]\explain{By \eqref{eq: def q hat}}\\
\leq& \frac{n_k}{n} \qty[\sqrt{\frac{\overline{\eta}_t}{\underline{\eta}_t}} \qty(\sum_{a} \pik_t(a | s)\sqrt{\hat{q}_{\pik, t}(s, a)} )^2    - { v_{\pik, t}(s)}^2] \explain{By \eqref{eq: variance upper bound of mu * rl}} \\
\leq&\E_{A_t \sim \pik_t}\left[\hat q_{\pik, t}(S_t, A_t) \mid S_t=s\right] - { v_{\pik, t}(s)}^2\explain{By \eqref{eq: average sufficient used form rl}}\\
=&\V\qty( \pdisgpik\qty(\tau^{\pik_{t:T-1}}_{t:T-1}) \mid S_t =s  ) .\explain{Lemma~\ref{lemma: recursive variance  
onpolicy rl}}\\
\end{align}
For $t\in[T-2]$,
\begin{align}
&\frac{n_k}{n}\V\left(\pdisgpik\qty(\tau^{\hat\mu_{t:T-1}}_{t:T-1}) \mid S_t=s\right) \\
=& \frac{n_k}{n} \E_{A_t\sim \hat\mu_t}\left[{\rho^{\pik, \hat\mu}}^2 \left(\E_{S_{t+1}}\left[\V\left(\pdisgpik\qty(\tau^{\hat\mu_{t+1:T-1}}_{t+1:T-1})\mid S_t\right) \mid S_t, A_t\right] + \nu_{\pik,t}(S_t, A_t) + q_{\pik, t}(S_t, A_t)^2\right) \mid S_t\right] \\
&-\frac{n_k}{n} { v_{\pik, t}(S_t)}^2\explain{Lemma~\ref{lemma: recursive-var}} \\
=&  \E_{A_t\sim \hat\mu_t}\left[{\rho^{\pik, \hat\mu}}^2\E_{S_{t+1}}\left[\frac{n_k}{n}\V\left(\pdisgpik\qty(\tau^{\hat\mu_{t+1:T-1}}_{t+1:T-1})\mid S_t\right) \mid S_t, A_t\right] \mid S_t\right] \\
&+ \frac{n_k}{n} \E_{A_t\sim \hat\mu_t}\left[{\rho^{\pik, \hat\mu}}^2 \left(\nu_{\pik,t}(S_t, A_t) + q_{\pik, t}(S_t, A_t)^2\right) \mid S_t\right] -\frac{n_k}{n} { v_{\pik, t}(S_t)}^2 \explain{Linearity of Expectation}\\
\leq&  \E_{A_t\sim \hat\mu_t}\left[{\rho^{\pik, \hat\mu}}^2 \left(\E_{S_{t+1}}\left[\V\left(\pdisgpik\qty(\tau^{\pik_{t+1:T-1}}_{t+1:T-1})\mid S_t\right) \mid S_t, A_t\right] \right) \mid S_t\right] \\
&+ \frac{n_k}{n} \E_{A_t\sim \hat\mu_t}\left[{\rho^{\pik, \hat\mu}}^2 \left(\nu_{\pik,t}(S_t, A_t) + q_{\pik, t}(S_t, A_t)^2\right) \mid S_t\right] -\frac{n_k}{n} { v_{\pik, t}(S_t)}^2 \explain{Indutive Hypothesis} \\
=&  \E_{A_t\sim \hat\mu_t}\left[{\rho^{\pik, \hat\mu}}^2 \left(\E_{S_{t+1}}\left[\V\left(\pdisgpik\qty(\tau^{\pik_{t+1:T-1}}_{t+1:T-1})\mid S_t\right) \mid S_t, A_t\right] +\nu_{\pik,t}(S_t, A_t) + q_{\pik, t}(S_t, A_t)^2\right) \mid S_t\right] \\
&+ (\frac{n_k}{n}-1) \E_{A_t\sim \hat\mu_t}\left[{\rho^{\pik, \hat\mu}}^2 \left(\nu_{\pik,t}(S_t, A_t) + q_{\pik, t}(S_t, A_t)^2\right) \mid S_t\right] -\frac{n_k}{n} { v_{\pik, t}(S_t)}^2 \\
=& \E_{A_t\sim \hat\mu_t}\left[{\rho^{\pik, \hat\mu}}^2  \hat q_{\pik, t}(S_t, A_t) \mid S_t\right] \\
&- \frac{n - n_k}{n} \E_{A_t\sim \hat\mu_t}\left[{\rho^{\pik, \hat\mu}}^2 \left(\nu_{\pik,t}(S_t, A_t) + q_{\pik, t}(S_t, A_t)^2\right) \mid S_t\right] -\frac{n_k}{n} { v_{\pik, t}(S_t)}^2\explain{By \eqref{eq: def q hat}}\\
\leq& \sqrt{\frac{\overline{\eta}_t}{\underline{\eta}_t}} \qty(\sum_{a} \pik_t(a|S_t)\sqrt{\hat q_{\pik,t}(a|S_t)})^2 \\
&- \frac{n - n_k}{n} \E_{A_t\sim \hat\mu_t}\left[{\rho^{\pik, \hat\mu}}^2 \left(\nu_{\pik,t}(S_t, A_t) + q_{\pik, t}(S_t, A_t)^2\right) \mid S_t\right] -\frac{n_k}{n} { v_{\pik, t}(S_t)}^2 \explain{by \eqref{eq: variance upper bound of mu * rl}}  \\
=& \sqrt{\frac{\overline{\eta}_t}{\underline{\eta}_t}} \qty(\sum_{a} \pik_t(a|S_t)\sqrt{\hat q_{\pik,t}(a|S_t)})^2 \\
&- \frac{n - n_k}{n}\qty( \E_{A_t\sim \hat\mu_t}\left[{\rho^{\pik, \hat\mu}}^2 \left(\nu_{\pik,t}(S_t, A_t) + q_{\pik, t}(S_t, A_t)^2\right)\mid S_t\right]  -v_{\pik, t}(S_t)^2)-{ v_{\pik, t}(S_t)}^2 \\
\leq& \E_{A_t \sim \pik_t}\left[\hat q_{\pik, t}(S_t, A_t) \mid S_t\right]  -{ v_{\pik, t}(S_t)}^2\explain{By \eqref{eq: sufficient delta used}}\\
=&\V\left(\pdisgpik\qty(\tau^{\pik_{t:T-1}}_{t:T-1})\mid S_t \right).\explain{Lemma~\ref{lemma: recursive variance  
onpolicy rl}}
\end{align}
\end{proof}
Now, we are ready to present the proof of Theorem~\ref{theorem: better than each average rl}.
\begin{align}
&\V\left(E^{\text{off},\pik}_{0:T-1} \right) \\
=&\frac{1}{n}\E_{S_0}\qty[\V\qty(\pdisgpik\qty(\tau^{\hat\mu_{0:T-1}}_{0:T-1}) \mid S_0 = s) ] + \frac{1}{n}\V_{S_0}\qty(v_{\pi,0}(S_0) ) \explain{By \eqref{eq: variance of E off}}\\
\leq&\frac{1}{n}\E_{S_0}\qty[\V\qty(\pdisgpik\qty(\tau^{\hat\mu_{0:T-1}}_{0:T-1}) \mid S_0 = s) ] + \frac{1}{n_k}\V_{S_0}\qty(v_{\pi,0}(S_0) ) \explain{$n_k\leq n$}\\
=&\frac{1}{n_k}\cdot\frac{n_k}{n}\V\left(\pdisgpik\qty(\tau^{\hat\mu_{0:T-1}}_{0:T-1}) \mid S_0=s\right)+ \frac{1}{n_k}\V_{S_0}\qty(v_{\pi,0}(S_0) )\\
\leq &\frac{1}{n_k}\V\left(\pdisgpik\qty(\tau^{\pik_{0:T-1}}_{0:T-1})\mid S_0=s \right)+ \frac{1}{n_k}\V_{S_0}\qty(v_{\pi,0}(S_0) )\explain{Lemma~\ref{lemma: nk n variance}}\\
=&\V\left(E^{\text{on},\pik}_{0:T-1} \right).\explain{By \eqref{eq: variance of E on}}
\end{align}

\subsection{Proof of Theorem \ref{theorem: better than each rl}}\label{append: better than each rl}
\begin{proof}
We prove it using induction.
When $t=T-1$,
$\forall k$, $\forall s$,
\begin{align}
 &\V\left(\pdisgpik\qty(\tau^{\hat \mu_{t:T-1}}_{t:T-1})\mid S_t=s\right)\\
  =&\E_{A_t \sim \mu_t}\left[{\rhok_t}^2 q_{\pik, t}(S_t, A_t)^2 \mid S_t\right] - { v_{\pik, t}(S_t)}^2\explain{Lemma~\ref{lemma: recursive-var}}\\
 =&\E_{A_t \sim \mu_t}\left[{\rhok_t}^2 
 \hat q_{\pik, t}(S_t, A_t) \mid S_t\right] - { v_{\pik, t}(S_t)}^2\explain{By \eqref{eq: def q hat}}\\
 \leq & \sqrt{\frac{\overline{\eta}_t}{\underline{\eta}_t}} \qty(\sum_{a} \pik_t(a|S_t)\sqrt{\hat q_{\pik,t}(S_t,a)})^2  - { v_{\pik, t}(S_t)}^2 \explain{By \eqref{eq: variance upper bound of mu * rl}}\\
 \leq&\E_{A_t \sim \pik_t}\left[\hat q_{\pik, t}(S_t, A_t) \mid S_t\right] - { v_{\pik, t}(S_t)}^2\explain{By \eqref{eq: eta sufficient each rl}}\\
 =&\V\left(\pdisgpik\qty(\tau^{\pik_{t:T-1}}_{t:T-1})\mid S_t \right).\explain{Lemma~\ref{lemma: recursive variance  
onpolicy rl}}
\end{align}
When $t\in[T-2]$, 
\begin{align}
 &\V\left(\pdisgpik\qty(\tau^{\hat \mu_{t:T-1}}_{t:T-1})\mid S_t=s\right)\\
=& \E_{A_t\sim \mu_t}\left[{\rhok_t}^2 \left(\E_{S_{t+1}}\left[\V\left(\pdisgpik\qty(\tau^{\hat \mu_{t+1:T-1}}_{t+1:T-1})\mid S_{t+1}\right) \mid S_t, A_t\right] + \nu_{\pik,t}(S_t, A_t) + q_{\pik, t}(S_t, A_t)^2\right) \mid S_t\right] \\
&-{ v_{\pik, t}(S_t)}^2\explain{Lemma~\ref{lemma: recursive-var}}\\
\leq &\E_{A_t\sim \mu_t}\left[{\rhok_t}^2 \left(\E_{S_{t+1}}\left[\V\left(\pdisgpik\qty(\tau^{\pik_{t+1:T-1}}_{t+1:T-1})\mid S_{t+1} \right)\mid S_t, A_t\right] + \nu_{\pik,t}(S_t, A_t) + q_{\pik, t}(S_t, A_t)^2\right) \mid S_t\right] \\
&-{ v_{\pik, t}(S_t)}^2\explain{Inductive Hypothesis}\\
=& \E_{A_t\sim \mu_t}\left[{\rhok_t}^2  \hat q_{\pik, t}(S_t, A_t) \mid S_t\right] -{ v_{\pik, t}(S_t)}^2\explain{By \eqref{eq: def q hat}}\\
 \leq & \sqrt{\frac{\overline{\eta}_t}{\underline{\eta}_t}} \qty(\sum_{a} \pik_t(a|S_t)\sqrt{\hat q_{\pik,t}(S_t,a)v})^2  - { v_{\pik, t}(S_t)}^2 \explain{By \eqref{eq: variance upper bound of mu * rl}}\\
\leq&\E_{A_t \sim \pik_t}\left[\hat q_{\pik, t}(S_t, A_t) \mid S_t\right] - { v_{\pik, t}(S_t)}^2\explain{By \eqref{eq: eta sufficient each rl}}\\
=&\V\left(\pdisgpik\qty(\tau^{\pik_{t:T-1}}_{t:T-1})\mid S_t \right).\explain{Lemma~\ref{lemma: recursive variance  
onpolicy rl}}
\end{align}
\end{proof}

\section{Experiment Details}
\subsection{Learning Closed-Form Behavior Policy}
In this section, we present an efficient algorithms
to learn the closed-form optimal behavior policy $\hat \mu$ with previously logged offline data.
By \eqref{def hat mu}, $\hat\mu$ is defined as $\textstyle \hat \mu_t(a|s) \propto 
\sqrt{\sum_{k=1}^K {\pik_t}(a|s) \hat q_{\pik, t}(s, a)^2}$,
where for each target policy $k$, $\hat q_{\pik}$ is defined in \eqref{eq: def q hat} as 
\begin{align}
\hat q_{\pik, t}(s, a) \doteq &q_{\pik, t}(s, a)^2 + \nu_{\pik,t}(s, a) \textstyle +\sum_{s'} p(s'|s, a)\V\left(\pdisgpik\qty(\tau^{\pik_{t+1:T-1}}_{t+1:T-1}) \mid S_{t+1} \right).
\end{align}
Learning $\hat \mu$ from this perspective is very inefficient because it requires approximations of the complex variance term $\textstyle\V\left(\pdisgpik\qty(\tau^{\pik_{t+1:T-1}}_{t+1:T-1}) \mid S_{t+1} \right)$ regarding future trajectory. To solve this problem, we restate
the recursive expression of $\hat q$ in the form of a Bellman equation \citep{variance2016Tamar,o2017uncertainty,sherstan2018directly} from \citet{liu2024efficient} and give its proof.
\begin{theorem}\label{lemma: hat-q-recursive}
For any target policy $\pik$, define
\begin{align}
  \hat{r}_{\pik,t}(s,a) \doteq 2r(s, a) q_{\pik, t}(s, a)- r^2(s, a).\label{def: hat r}
\end{align}
Then $\hat{q}_{\pik, t}(s, a) = \hat r_{\pik, t}(s, a)$ for $t=T-1$ 
and otherwise
\begin{align}
  \hat{q}_{\pik, t}(s, a) = \textstyle \hat{r}_{\pik,t}(s,a) +  \sum_{s', a'} p(s'|s, a) \pik_{t+1}(a'|s') \hat{q}_{\pik, t+1}(s', a'). \label{eq: recursive hat q}
\end{align}
\end{theorem}
\begin{proof}
For any $k$, for $t = T-1$, we have
\begin{align}
 \hat q_{\pik, t}(s, a) &=q_{\pik, t}(s, a)^2 \explain{Definition of $\hat q_{\pik, t}$~\eqref{eq: def q hat}} \\
&= \hat{r}_{\pik,t}(s,a). \explain{By $q_{\pik, T-1}(s, a) = r(s, a)$ and \eqref{def: hat r}}
\end{align}
For $t\in [T-2]$, we have
\begin{align}
&\hat{q}_{\pik,t}(s,a) \\
=& q_{\pik, t}(s, a)^2 + \nu_{\pik,t}(s, a)+\sum_{s'} p(s'|s, a)\V\left(\pdisgpik\qty(\tau^{\pik_{t+1:T-1}}_{t+1:T-1}) \mid S_{t+1} = s'\right)\explain{Definition of $\hat{q}$ \eqref{eq: def q hat}} \\
=& q_{\pik, t}(s, a)^2 + \nu_{\pik,t}(s, a)+\sum_{s'} p(s'|s, a)\qty(\E_{A_{t+1} \sim \pik_{t+1}}\left[\hat q_{\pik, t+1}(S_{t+1}, A_{t+1}) \mid S_{t+1}=s'\right] -  v_{\pik, t+1}(S_{t+1})^2)\explain{Lemma~\ref{lemma: recursive variance  
onpolicy rl}} \\
=& \nu_{\pik,t}(s, a) + q_{\pik, t}(s, a) ^2- \sum_{s'} p(s'|s, a)  v_{\pik, t+1}(s')^2 +  \sum_{s', a'} p(s'|s, a) \pik_{t+1}(a'|s') \hat{q}_{\pik, t+1}(s', a') \\
=& -(\E [ v_{\pik, t+1}(S_{t+1})\mid S_t=s, A_t=a ])^2 + q_{\pik, t}(s, a)^2 + \sum_{s', a'} p(s'|s, a) \pik_{t+1}(a'|s') \hat{q}_{\pik, t+1}(s', a')  \explain{Definition of $\nu$ \eqref{def:nu}} \\
=& -(q_{\pik, t}(s, a) - r(s, a) )^2 + q_{\pik, t}(s, a)^2 + \sum_{s', a'} p(s'|s, a) \pik_{t+1}(a'|s') \hat{q}_{\pik, t+1}(s', a') \\
=&  2r(s, a) q_{\pik, t}(s, a)- r(s, a)^2  + \sum_{s', a'} p(s'|s, a) \pik_{t+1}(a'|s') \hat{q}_{\pik, t+1}(s', a') \\
=&  \hat r_{\pik, t}(s, a)  + \sum_{s', a'} p(s'|s, a) \pik_{t+1}(a'|s') \hat{q}_{\pik, t+1}(s', a'),
\end{align}
which completes the proof.
\end{proof}

\emph{This derivation enables the implementation of any off-the-shelf offline policy evaluation methods to learn $\hat q_{\pik}$,}
after which the behavior policy $\hat \mu$ can be computed easily with~\eqref{def hat mu}.
For generality,
we consider
the behavior policy agnostic offline learning setting \citep{nachum2019dualdice},
where the offline data in the form of
  $\qty{(t_i,s_i,a_i,r_i,s_i')}_{i=1}^m$
consists of $m$ previously logged data tuples.
In the $i$-th data tuple, 
$t_i$ is the time step, 
$s_i$ is the state at time step $t_i$, 
$a_i$ is the action executed on state $s_i$, 
$r_i$ is the sampled reward, 
and $s_i'$ is the successor state. 
Those tuples can be generated by one or more, known or unknown behavior policies.
Those tuples do not need to form a complete trajectory.

In this work,
we use 
Fitted $Q$-Evaluation (FQE, \citet{le2019batch}) as a demonstration, but our algorithm can incorporate any state-of-the-art offline policy evaluation methods to approximate $\hat q_\pik$.
To learn $\hat r_\pik$,
it is sufficient to learn $q$, in which 
FQE can be applied.
Then, FQE is invoked to learn an approximation of $\hat q_\pik$.
We refer the reader to Algorithm \ref{alg: ODI algorithm} for a detailed exposition of our algorithm.
In practice, 
we split the offline data into training sets and test sets to tune all the hyperparameters offline
in Algorithm \ref{alg: ODI algorithm}.

\subsection{GridWorld}
For a Gridworld with size $m^3$, we set its width, height, and time horizon $T$ all to be $m$. We test Gridworlds with $m^3=1,000$ and $m^3=27,000$ states. The action space contains four possible actions: up, down, left, and right. After taking an action, the agent has a probability of $0.9$ to move accordingly and a probability of $0.1$ to move uniformly at random. If running into a boundary, the agent stays in the current position. 
The reward function $r(s,a)$ is randomly generated.

We generate target policies using the proximal policy optimization (PPO) algorithm \citep{schulman2017proximal} with the default parameters in CleanRL \citep{huang2022cleanrl}. We choose PPO just for a demonstration. Our method copes with any other deep RL algorithm. We randomly draw $10$ policies in a randomly chosen time step interval. We obtain the ground truth policy performance for 
each target policy by executing on-policy Monte Carlo evaluation for $10^6$ total episodes. 
Our offline dataset includes $10^4$ episodes from various policies with a wide range of performances.
We execute Algorithm~\ref{alg: ODI algorithm} to learn our tailored behavior policy.
When approximating $q$ and $\hat q$, we use Fitted Q-Evaluation \citep{le2019batch}. We use a one-hidden-layer neural network for Fitted Q-Evaluation.
We test the neural network size for Fitted Q-Evaluation with $[64,128,256]$ and choose $64$ as the final size. 
We test the learning rate for Adam optimizer with $[1\text{e}^{-5},1\text{e}^{-4},1\text{e}^{-3},1\text{e}^{-2}]$ and choose to use the default learning rate $1\text{e}^{-3}$ as learning rate for Adam optimizer \citep{kingma2014adam}.

The results for each target policy are shown in Figure~\ref{fig:gridworld detail 10} and Figure~\ref{fig:gridworld detail 30} in terms of the \textit{relative error} against total \textit{samples}, as described in the main text. Notably, for the On-policy Monte Carlo estimator and the ODI estimator \citep{liu2024efficient}, samples for each single target policy $\pik$ are collected once in every $K=10$ total sample steps. 
Smooth lines are plotted through interpolation.
Each line in Figure~\ref{fig:gridworld detail 10} and Figure~\ref{fig:gridworld detail 30} is averaged over 900 different runs (30 groups of target policies, each having 30 independent runs), indicating strong statistical significance.
Our method (MPE) \textit{consistently outperforms} all other estimators for the evaluation of \textit{every} single target policy, demonstrating \textit{state-of-the-art performance}.

\!\!
\begin{figure}[ht]
\includegraphics[width=1\textwidth]{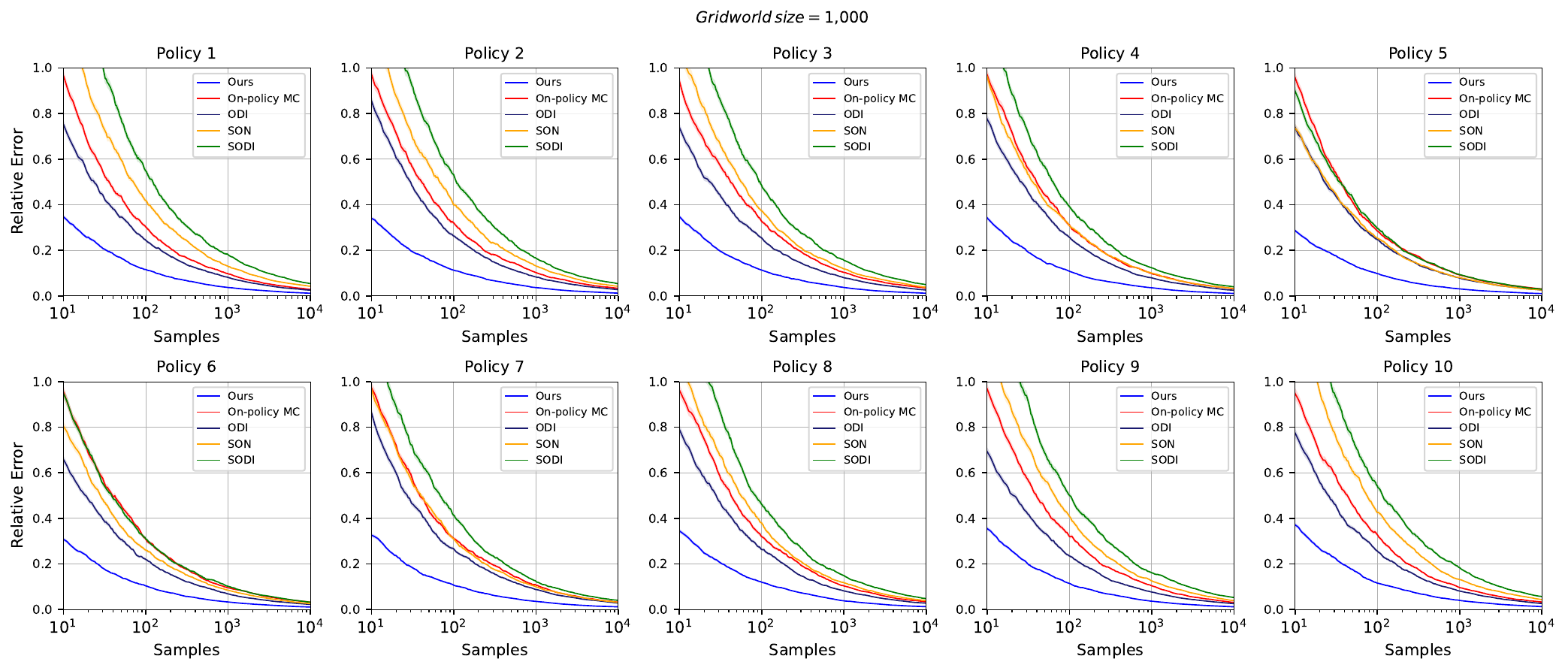}
\centering
\caption{Results on Gridworld. 
Each curve is averaged over 900 runs (the corresponding target policies from 30 groups, each having 30 independent runs). 
Shaded regions denote standard errors and are invisible for some curves because they are too small.}
\label{fig:gridworld detail 10}
\end{figure}

\begin{figure}[H]
\includegraphics[width=1\textwidth]{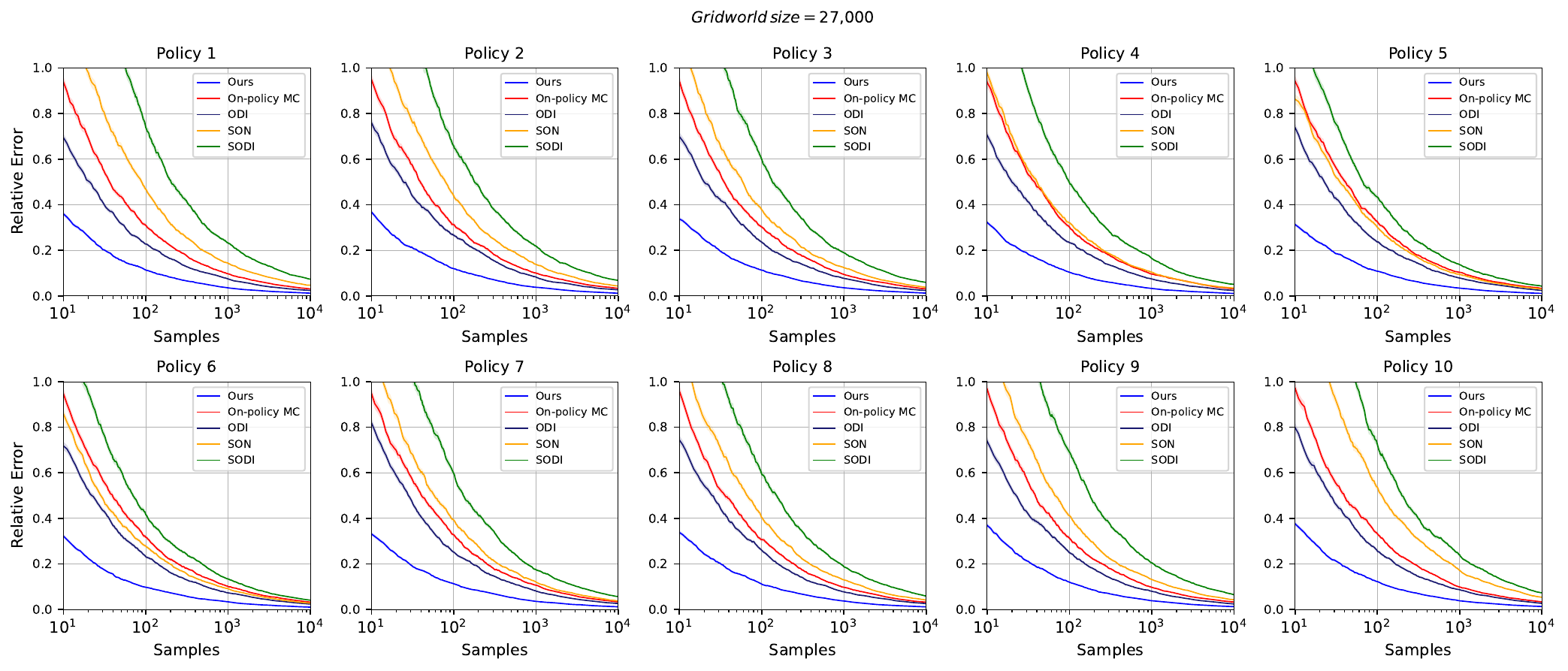}
\centering
\caption{
Results on Gridworld. 
Each curve is averaged over 900 runs (the corresponding target policies from 30 groups, each having 30 independent runs). 
Shaded regions denote standard errors and are invisible for some curves because they are too small.
}
\label{fig:gridworld detail 30}
\end{figure}

\subsection{MuJoCo}
\begin{figure}[ht]
\begin{minipage}{0.18\textwidth}
\centering
\includegraphics[width=1\textwidth]{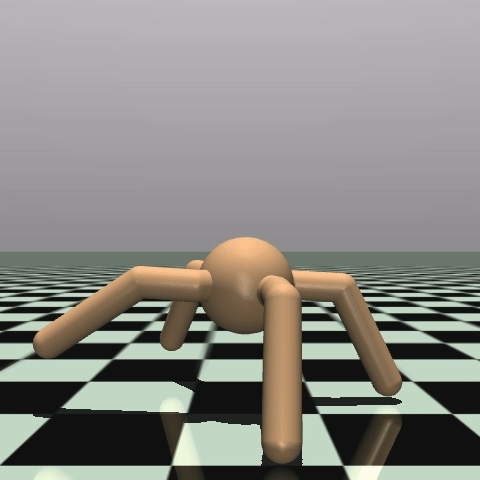}
\end{minipage}
\begin{minipage}{0.18\textwidth}
\centering \includegraphics[width=1\textwidth]{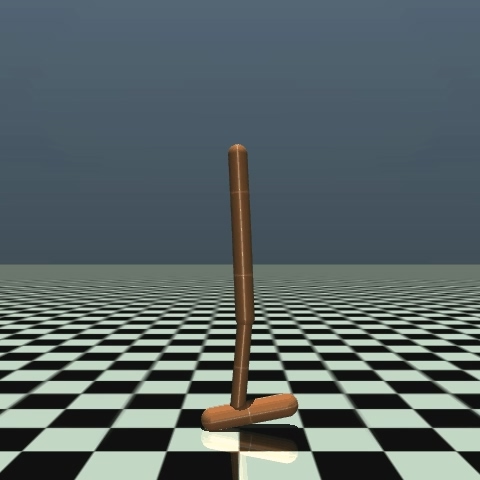}
\end{minipage}
\begin{minipage}{0.18\textwidth}
\centering \includegraphics[width=1\textwidth]{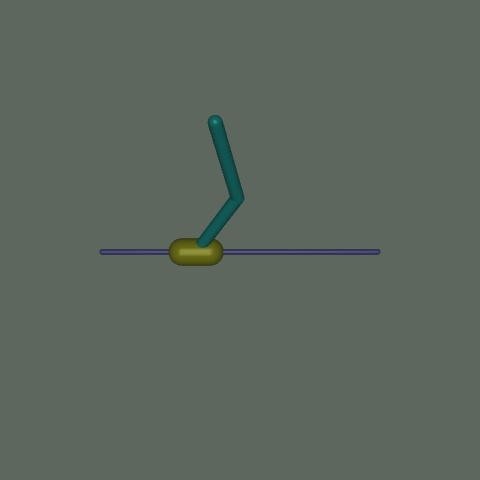}
\end{minipage}
\begin{minipage}{0.18\textwidth}
\centering \includegraphics[width=1\textwidth]{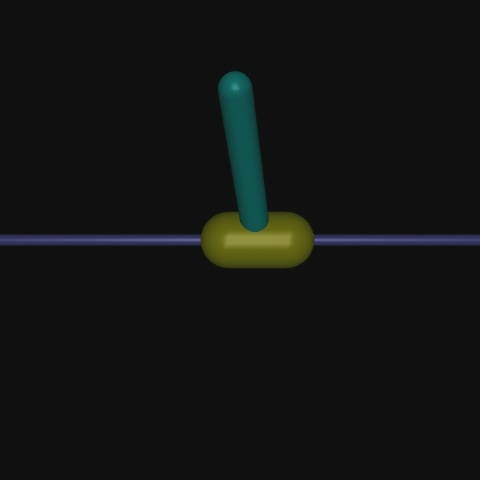}
\end{minipage}
\begin{minipage}{0.18\textwidth}
\centering \includegraphics[width=1\textwidth]{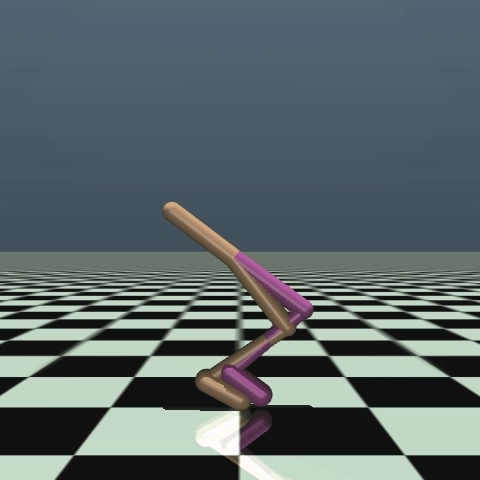}
\end{minipage}
\centering
\caption{
MuJoCo robot simulation tasks \citep{todorov2012mujoco}. Pictures are adapted from \citep{liu2024efficient}.
Environments from the left to the right are Ant, Hopper, InvertedDoublePendulum,  InvertedPendulum, and Walker.
} 
\label{fig:cart_pole_image}
\end{figure}
In experiments of MuJoCo robot simulation tasks \citep{todorov2012mujoco}, we use the same method for obtaining offline data, randomly generating target policies, training behavior policies, and applying the same hyperparameters as in the Gridworld experiment.
We discretize the first dimension of MuJoCo action space in our experiment. The policies of remaining dimensions are obtained during PPO \citep{schulman2017proximal} training process, and deemed as part of the environment.
The following table offers an additional interpretation to Figure \ref{fig:mujoco}.
\begin{table}[H]
    \centering
\begin{tabular}{llllll}
\toprule
Env ID & \textbf{Ours} & On-policy MC & ODI & SON & SODI \\
\midrule
Ant & \textbf{0.115} & 1.000 & 0.606 & 1.144 & 1.548 \\
Hopper & \textbf{0.114} & 1.000 & 0.580 & 1.287 & 1.413 \\
InvertedDoublePendulum & \textbf{0.111} & 1.000 & 0.494 & 0.882 & 1.582 \\
InvertedPendulum & \textbf{0.124} & 1.000 & 0.565 & 0.889 & 1.250 \\
Walker & \textbf{0.094} & 1.000 & 0.590 & 0.759 & 1.056 \\
\bottomrule
\end{tabular}
\caption{Relative variance of estimators on MuJoCo environments. The relative variance is defined as the variance of each estimator divided by the variance of the on-policy Monte Carlo estimator. Numbers are averaged over 900 independent runs (30 groups of target policies, each having 30 independent runs).}
\label{table: mujoco variance}
\end{table}

\begin{table}[H]
    \centering
\begin{tabular}{llllll}
\toprule
Env ID & \textbf{Ours} & On-policy MC & ODI & SON & SODI \\
\midrule
Ant & \textbf{125} & 1000 & 626 & 1171 & 1701 \\
Hopper & \textbf{115} & 1000 & 583 & 1253 & 1413 \\
InvertedDoublePendulum & \textbf{103} & 1000 & 464 & 835 & 1547 \\
InvertedPendulum & \textbf{111} & 1000 & 530 & 916 & 1166 \\
Walker & \textbf{97} & 1000 & 542 & 749 & 1033 \\
\bottomrule
\end{tabular}
\caption{Episodes needed to achieve the same of estimation accuracy that on-policy Monte Carlo achieves with $1000$ episodes on MuJoCo environments. Numbers are averaged over 900 independent runs (30 groups of target policies, each having 30 independent runs) and their standard errors are shown in Figure \ref{fig:mujoco}.}
\label{table: mujoco samples}
\end{table}

\end{document}